\documentclass[10pt]{article} 

\usepackage[accepted]{rlj} 

\usepackage{booktabs}
\usepackage{float}
\usepackage{amssymb}            
\usepackage{mathtools}          
\usepackage{mathrsfs}           
\usepackage{graphicx}           
\usepackage{subcaption}         
\usepackage[space]{grffile}     
\usepackage{url}                
\usepackage{lipsum}             
\usepackage{amsthm}
\newtheorem{definition}{Definition}
\usepackage{wrapfig}
\usepackage{amssymb}
\usepackage{amsthm}
\newtheorem{theorem}{Theorem}[section]

\newtheorem{lemma}[theorem]{Lemma}
\usepackage{hyperref} 

\numberwithin{equation}{section}  

\title{\centering{Towards Improving Reward Design in RL: \\ A Reward Alignment Metric for RL Practitioners}}

\setrunningtitle{Towards Improving Reward Design in RL: A Reward Alignment Metric for RL Practitioners}

\author{Calarina Muslimani\textsuperscript{1, $\dagger$}, Kerrick Johnstonbaugh\textsuperscript{2}, Suyog Chandramouli\textsuperscript{3}, \\ Serena Booth\textsuperscript{4, 5},  W. Bradley Knox\textsuperscript{6}, Matthew E. Taylor\textsuperscript{1,7}}

\emails{ $^\dagger$musliman@ualberta.ca}
\affiliations{
$^{1}$\textbf{University of Alberta}\\
$^{2}$\textbf{RL Core Technologies}\\
$^{3}$\textbf{Princeton University}\\
$^{4}$\textbf{Brown University}\\
$^{5}$\textbf{Simons Institute for the Theory of Computing}\\
$^{6}$\textbf{University of Texas at Austin}\\
$^{7}$\textbf{Alberta Machine Intelligence Institute (Amii)}\\
}

\contribution{
    This paper introduces the Trajectory Alignment Coefficient, a metric to evaluate reward alignment---the extent to which a reward function encodes the preferences of a human stakeholder. This metric only requires access to human preferences over trajectory distributions. 
    }
    {
Prior work \citep{epic, dard} has proposed reward distance metrics, which can be treated as reward alignment metrics if one of the reward functions accurately reflects the preferences of a human stakeholder. 
Other work \citep{value_alignment}
has proposed alignment verification of reward functions. However, these metrics assume access to ground-truth reward or value functions 
which our metric does not require.
    }
\contribution{
Through an $11$--person user study, we demonstrate that the Trajectory Alignment Coefficient can assist RL practitioners (self-identified) with reward selection, compared to relying solely on inspection of the reward function definition. Specifically, we found the following statistically significant results: (1) access to the Trajectory Alignment Coefficient during reward selection reduced perceived cognitive workload by \( 1.5\)x, (2) was preferred by \( 82\% \) of participants over the Reward Only condition, and (3) increased the success rate of selecting reward functions that produced performant policies by \( 41\% \) (compared to unselected alternative rewards).
    }
    {
    The user study results presented in this work are in the context of the Hungry-Thirsty domain \citep{singh2009rewards}. This test-bed has been shown to be particularly challenging for reward design \citep{perils_reward_design}. Further, we specifically do not include \cite{epic, dard, value_alignment} as conditions in the user study because they are not directly comparable due to the difference in necessary assumptions.} 

    \contribution{
    We prove that the Trajectory Alignment Coefficient is invariant to potential-based reward shaping and positive linear transformations of the reward function if and only if the metric considers trajectory distributions that share the same start state distribution.
    }
    {
    Invariance to these reward transformations is a common property of reward evaluation metrics proposed in previous works \citep{epic,dard}.
    }

\summary{
Reinforcement learning (RL) agents are fundamentally limited by the quality of the reward functions they learn from, yet reward design is often overlooked under the assumption that a well-defined reward is readily available. However, this is rarely the case in real-world applications, and reward design is a challenging endeavor: sparse rewards can hinder learning, while dense rewards may increase the risk of misspecification. Reward evaluation is equally problematic: \emph{how do we know if a reward function is correctly specified?} 
 In our work, we address this challenge by focusing on \textit{reward alignment} --- assessing whether a reward function accurately encodes the preferences of a human stakeholder.
As a concrete measure of reward alignment, we introduce the Trajectory Alignment Coefficient. This metric quantifies the similarity between a human stakeholder's ranking of trajectory distributions and those induced by a given reward function.
We validate the usefulness of this metric theoretically and through a user study.
}

\begin{document}

\makeCover  
\maketitle  

\begin{abstract}
Reinforcement learning agents are fundamentally limited by the quality of the reward functions they learn from, yet reward design is often overlooked under the assumption that a well-defined reward is readily available. 
However, in practice, designing rewards is difficult, and even when specified, evaluating their correctness is equally problematic: \emph{how do we know if a reward function is correctly specified?} 
 In our work, we address these challenges by focusing on \textit{reward alignment} --- assessing whether a reward function accurately encodes the preferences of a human stakeholder.
As a concrete measure of reward alignment, we introduce the Trajectory Alignment Coefficient to quantify the similarity between a human stakeholder's ranking of trajectory distributions and those induced by a given reward function.
We show that the Trajectory Alignment Coefficient exhibits desirable properties, such as not requiring access to a ground truth reward, invariance to potential-based reward shaping, and applicability to online RL.
Additionally, in an $11$--person user study of RL practitioners, we found that access to the Trajectory Alignment Coefficient during reward selection led to statistically significant improvements. Compared to relying only on reward functions, our metric reduced cognitive workload by $1.5$x, was preferred by $82\%$ of users and increased the success rate of selecting reward functions that produced performant policies by $41\%$.
\end{abstract}
\section{Introduction}\label{sec:intro}

In reinforcement learning (RL), the \textit{reward hypothesis} states that ``all of what we mean by goals and purposes can be well thought of as maximization of the expected value of the cumulative sum of a received scalar signal (reward)'' \citep{sutton2018reinforcement}. 
More generally, this means that RL agents can solve a task provided the reward function properly defines the task's objective. 
However, the reward hypothesis does not address the practical challenges of designing reward functions. 
In practice, reward design is often a difficult and error-prone process carried out by human engineers \citep{defining_reward_gaming, perils_reward_design, specifying_rl_objectives}.  

These challenges can become more pronounced in real-world RL applications, where reward design is typically a \textit{collaborative process} between RL practitioner(s) and domain expert(s). 
While the domain expert has specialized knowledge of the task, they typically lack RL expertise, making it difficult for them to define a reward function explicitly. Instead, the domain expert might express preferences, constraints, or desired outcomes, leaving the RL practitioner responsible for designing (or selecting) a reward function that satisfies these preferences. This collaboration can increase the complexity of crafting reward functions that correctly specify objectives. 

Sparse reward functions are conceptually simple to understand and implement but are less commonly used in practice since current RL algorithms struggle to learn from infrequent signals \citep{pignatelli2024surveytemporalcreditassignment}.
To overcome this, dense reward functions are employed, which provide more frequent feedback to help mitigate the credit assignment problem.
However, reward misspecification remains a challenge \citep{amodei2016concreteproblemsaisafety,defining_reward_gaming}. For example, a recent survey found that reward shaping, a method intended to facilitate learning, is commonly used in RL applications for autonomous driving \citep{reward_misdesign_AD}; without careful design, reward shaping can introduce unintended biases. This can result in RL agents exploiting shortcuts in the reward function or failing to achieve the 
``true'' task objective \citep{pan2022effectsrewardmisspecificationmapping}.
Such issues can pose serious safety risks in real-world applications like autonomous driving and industrial process control.

\emph{Reward evaluation} is also challenging. This is the process of assessing whether a reward function accurately captures the intended task.
A common approach is the ``rollout'' method, where a policy is trained to optimize the reward function, and then its rollouts are examined to assess the learned behavior \citep{perils_reward_design, epic}. 
However, this approach has several limitations: (1) it is computationally expensive, (2) can result in reward overfitting---where reward functions become unintentionally over-engineered for a specific algorithm or environment configuration---and (3) assumes that policies are evaluated outside of training, making it less applicable to the online RL setting.
Alternatively, prior works~\citep{epic, dard} have proposed distance metrics for reward evaluation but these require a ground-truth reward for baseline comparison, limiting their integration into the reward design pipeline (unless shaping a reward based on an existing function). 
Moreover, other metrics~\citep{specifying_rl_objectives,value_alignment} focus solely on alignment verification and do not measure partial alignment.


In this work, we focus on \emph{reward alignment} as a means of reward evaluation, which we define as the extent to which a reward function preserves human preferences.
To operationalize this concept, we introduce the \textit{Trajectory Alignment Coefficient} ($\sigma_{\textit{TAC}}$). This metric evaluates the similarity between a human stakeholder's preferences over trajectory distributions (of which trajectories are a special case) and those induced by a given reward, discount factor pair. 
It overcomes key limitations of previous work by eliminating the need for a ground-truth reward, instead relying on human preferences. 
 Unlike alignment verification, the Trajectory Alignment Coefficient measures the \textit{degree} of reward alignment, allowing it to distinguish between reward functions that yield the same optimal policy but rank intermediate trajectory distributions differently---making it suitable for online RL.
 
Additionally, we prove the necessary and sufficient conditions for the Trajectory Alignment Coefficient
to be invariant to common transformations, in particular potential-based shaping and positive linear rescaling. This invariance is important because sensitivity to these transformations can cause functionally equivalent rewards to receive different scores, leading to unreliable assessments.
Beyond reward alignment, the Trajectory Alignment Coefficient can serve as a distance metric for comparing reward functions (and their associated discounting). While our primary focus is on reward design with human preferences as the reference, this perspective highlights its potential as a tool for comparing reward functions more broadly.


Lastly, we assess whether the Trajectory Alignment Coefficient can aid RL practitioners in the reward design process (see Figure~\ref{fig:domain_expert_rl_expert_interaction}). Specifically, we investigate its benefit in reward selection---i.e., choosing performant reward functions that capture a domain expert’s preferences. 
To evaluate this, we conducted an $11$--person user study in the Hungry-Thirsty domain \citep{singh2009rewards}, a test-bed where RL practitioners have struggled to design well-specified rewards \citep{perils_reward_design}. 
Our statistically significant findings show that access to the Trajectory Alignment Coefficient during reward selection (1) reduced perceived cognitive workload by \( 1.5x \), (2) was preferred by \( 82\% \) of users over the Reward Only condition, and (3) increased the success rate of selecting reward functions that produced performant policies by \( 41\% \).
Ultimately, our work takes a step toward improving reward design and selection in RL by introducing a metric that measures the alignment between a proposed reward function and a set of human preferences over sampled trajectory distributions.

\begin{figure}[h!]
  \centering

  \hfill
  {\includegraphics[width=\textwidth, trim = 0cm 7cm 9cm 8cm, clip]{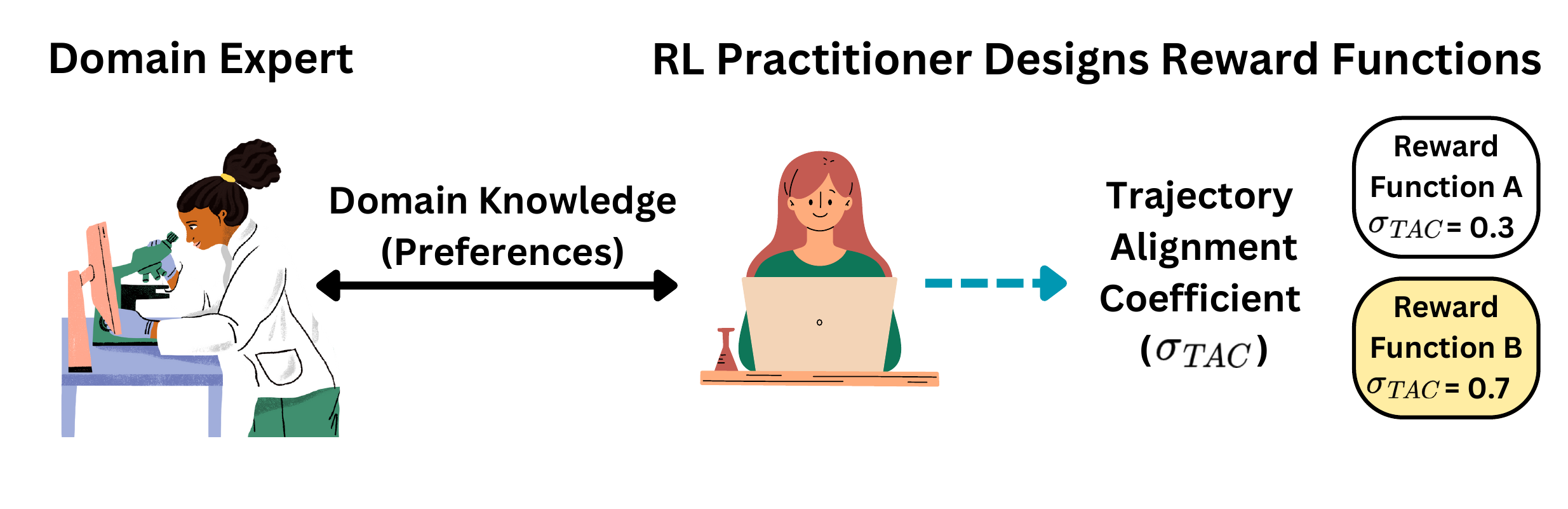}}

  \caption{This illustrates the interaction between a domain expert and an RL practitioner in real-world applications (black arrow) and how our metric, $\sigma_{\textit{TAC}}$, integrates into this process (blue arrow). 
  }
  
  \label{fig:domain_expert_rl_expert_interaction}
\end{figure}

\section{Related Work}\label{sec:related_work}
Early alignment research focused on directly training agents to align with human preferences \citep{hadfield2016cooperative}. However, these approaches did not include an assessment of the agent’s alignment.  
Recent efforts have shifted to evaluating the quality of engineered and learned reward functions. For example,~\cite{perils_reward_design,reward_misdesign_AD} have conducted empirical investigations to identify shortcomings in current reward design practices and evaluation schemes.
Furthermore, metrics have been proposed to compare reward functions without requiring policy evaluations (see Table~\ref{tab:distance_measures}). While some methods are invariant to potential-based shaping, they can rely on access to a ground-truth reward function \citep{dard, epic}, which is often impractical. 
Likewise, \citet{value_alignment} proposed verification methods to assess the alignment of an RL agent’s behavior but define alignment in terms of optimal policies. 
We argue, however, that defining alignment in this manner can be limiting, particularly in online RL where one cares about the agent's lifetime performance. 
Our work builds on \citet{specifying_rl_objectives}, which described methods to identify misalignment and outlined its common causes. However, this prior work focused on detecting whether misalignment exists, offering only a binary assessment. We extend this framework by introducing a real-valued metric that quantifies the degree of alignment, enabling a more nuanced evaluation of reward function quality.
 Lastly, while LLM alignment is also a prominent topic~\citep{shen2023largelanguagemodelalignment}, it is beyond the scope of this work due to its broad focus, which can include mitigating adversarial attacks and detecting bias. 

\begin{table}[h!]
\centering
\resizebox{\textwidth}{!}{%
\begin{tabular}{l c c c c c}  
\toprule
\textbf{\textsc{Metric}} & \textbf{Invariant} & \textbf{No GT \(r\)} & \textbf{Not} & \textbf{No Human} & \textbf{Suitable for} \\  
 &  & \textbf{Required} & \textbf{Binary} & \textbf{Preferences} & \textbf{Online RL} \\  
\midrule
\textsc{\cite{epic}} & \checkmark & \(\times\) & \checkmark & (\checkmark) & \checkmark \\
\textsc{\cite{dard}} & \checkmark & \(\times\) & \checkmark & (\checkmark) & \checkmark \\
\textsc{\cite{value_alignment}} & -- & \(\times\) & \(\times\) & (\checkmark) & \(\times\) \\
\textsc{\cite{specifying_rl_objectives}} & \checkmark & \checkmark & \(\times\) & \(\times\) &  \(\times\)   \\
\textsc{Trajectory Alignment Coefficient} & \checkmark & \checkmark & \checkmark & \(\times\) & \checkmark \\
\bottomrule
\end{tabular}
}
\caption{Comparison of reward evaluation measures. \(\checkmark\) indicates the metric satisfies the property, \(\times\) indicates it does not, (\checkmark) indicates partial satisfaction, and -- indicates the property was not evaluated.}
\label{tab:distance_measures}
\end{table}

\section{Background}\label{sec:background}
This section first provides background on RL, then discusses how a reward function and discount factor pair, ($r, \gamma$), induce preference orderings over trajectories (and trajectory distributions), a concept rooted in prior work on policy preferences \citep{bowling2023settling}.

\begin{definition}
A Markov decision process (MDP) is defined by the tuple 
$(\mathcal{S}, \mathcal{A}, r, p, \mu, \gamma)$, where $\mathcal{S}$ is the state space, $\mathcal{A}$ is the action space, $r: \mathcal{S} \times \mathcal{A} \times \mathcal{S} \rightarrow \mathbb{R}$ is the reward function and $p: \mathcal{S} \times \mathcal{A} \times \mathcal{S} \rightarrow [0,1]$ is the state transition function. The initial state distribution is given by $\mu$, and $\gamma \in [0,1)$ is the discount factor that controls the weighting of future rewards. 
\end{definition}

In RL, at every time-step $t$, the agent takes an action $a_t$ in state $s_t$, transitions to state $s_{t+1}$, and then receives reward $r_{t+1}$. 
A trajectory \(\tau\) is a sequence of ($s_t, a_t, s_{t+1}$) tuples
that either reaches a terminal state after a finite number of steps or continues indefinitely.
The return of a trajectory is defined as the sum of discounted future rewards,
$G_{r}(\tau) = \sum_{t=0}^{T} \gamma^t r_{t+1}$,
where \( T = |\tau|-1 \) for episodic tasks or \( T \to \infty \) for continuing tasks.
The agent attempts to learn a policy, $\pi:\mathcal{S} \times \mathcal{A} \rightarrow [0,1]$ to maximize the expected return.

\paragraph{Reward Functions Induce Preference Orderings}\label{sec:reward_funcs_induce_preferences}
Consider the deterministic case where we define preferences over a set of trajectories that share the same start state.
In this case, given (\(r, \gamma\)) a preferred trajectory is one that yields a greater return:
\begin{equation}\label{eq:determinstic_traj}
\tau_A \underset{(r, \gamma)}{\succsim} \tau_B \iff G_r(\tau_\text{A}) \ge  G_r(\tau_\text{B})
\end{equation}
where \( \tau_A \underset{(r, \gamma)}{\succ} \tau_B \) indicates that trajectory \( \tau_A \) is preferred over \( \tau_B \) with respect to \( (r,  \gamma)\). 

We now shift to the stochastic setting, which arises when the environment or the agent's behavior is stochastic. In this case, we consider probability distributions over trajectories. 
Note, we specifically focus on trajectory distributions rather than policies, as some distributions (e.g., those that are generated from non-Markovian policies) cannot correspond to any Markov policy.

\begin{definition}\label{Trajectory_Distributions}
Let \(H(\mu)\) be the set of all probability distributions over trajectories that share the same initial state distribution \(\mu\). That is,  
\[
H(\mu) = \left\{ \eta(\tau) \mid \eta(\tau) = \mu(s_0) P(a_0, s_1, a_1, \ldots \mid s_0) \right\} ,
\]
where \(P(a_0, s_1, a_1, \ldots \mid s_0)\) is an arbitrary conditional distribution over trajectories given the initial state \(s_0\). We refer to \(\eta(\tau) \in H(\mu)\) as a trajectory distribution and omit the explicit dependence on \(\tau\) for brevity.   
\end{definition}


\begin{definition}\label{def:preferences_trajectory_distribution}
 Given $(r, \gamma)$, we define a preference ordering over trajectory distributions as follows:
\begin{equation}\label{eq:stochastic_traj}
\eta_\text{A} \underset{(r, \gamma)}{\succsim} \eta_\text{B}   \iff \mathbb{E}_{\tau_A\sim\eta_\text{A}}[G_r(\tau_A)] \ge \mathbb{E}_{\tau_B\sim\eta_\text{B}}[G_r(\tau_B)] 
\end{equation}
\end{definition}

Equations \eqref{eq:determinstic_traj} and \eqref{eq:stochastic_traj} imply that the $(r, \gamma)$ pair naturally induce a preference ordering over trajectories (or trajectory distributions) via the expected return. 
To illustrate these concepts, consider the simple autonomous driving task \citep{specifying_rl_objectives} in \mbox{Figure \ref{fig:AD_alignnment_example}}. Suppose there exists only three trajectories $\{\tau_{\text{{success}}}, \tau_{\text{{idle}}}, \tau_{\text{{crash}}}\}$, and a trajectory distribution $\eta_{\text{success-crash}}$. $\tau_{\text{{success}}}$ consists of safe driving.  $\tau_{\text{{crash}}}$ consists of a car crashing and $\tau_{\text{{idle}}}$ consists of a car remaining parked. $\eta_{\text{success-crash}}$ is a trajectory distribution that places 90\% of its probability mass on $\tau_{\text{success}}$ and 10\% on $\tau_{\text{crash}}$.
Next, consider the pair $(r, \gamma)$ with return values: $G_r(\tau_{\text{{success}}})=10$, $G(\tau_{\text{{idle}}})=0$, $G_r(\tau_{\text{{crash}}})=-50$. By the probabilities of $\eta_{\text{success-crash}}$, $\mathbb{E}_{\tau\sim\eta_{\text{success-crash}}}[G_r(\tau)]=4$. 
Based on equations \eqref{eq:determinstic_traj} and \eqref{eq:stochastic_traj}, the resulting preference ordering is 
$\tau_{\text{{success}}}  \succ \eta_{\text{success-crash}} \succ \tau_{\text{{idle}}} \succ \tau_{\text{{crash}}}$.
\begin{figure}[h!]
  \centering

  \hfill
  {\includegraphics[width=\textwidth, trim = 0cm 2cm 0cm 4cm, clip]{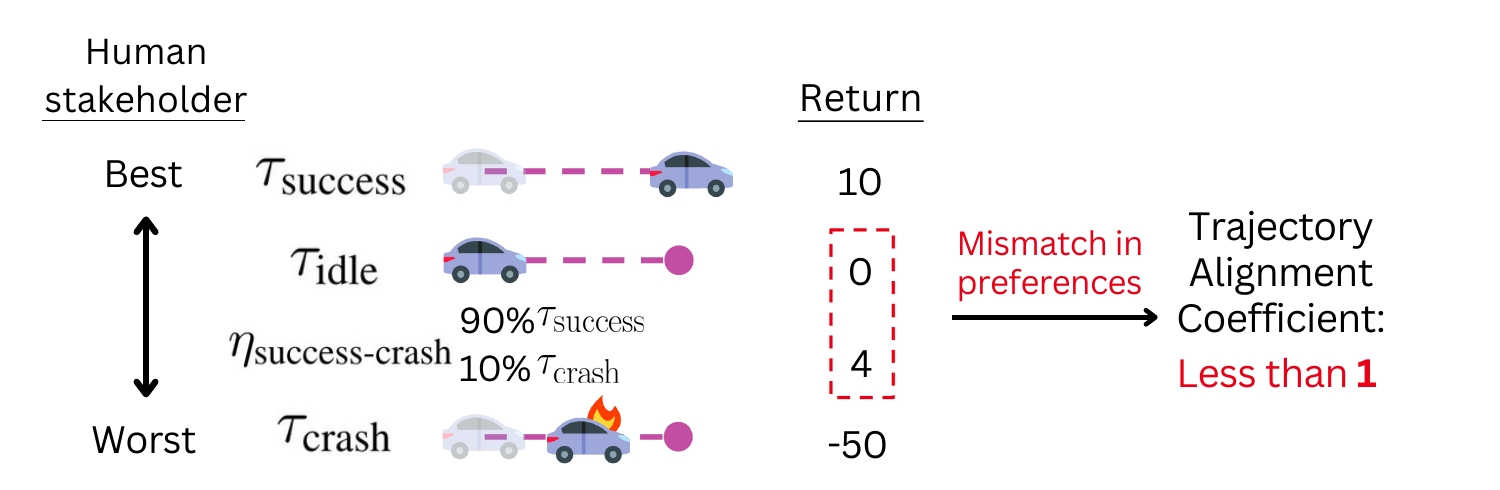}}

  \caption{This provides an example of the Trajectory Alignment Coefficient in the case of a simple autonomous driving scenario.}
  
  \label{fig:AD_alignnment_example}
\end{figure}

\section{An Alignment Metric for Reward Function Evaluation}\label{sec:alignment_method}

This section introduces the Trajectory Alignment Coefficient as a reward alignment metric and establishes its key theoretical properties.

\subsection{Trajectory Alignment Coefficient}
To establish a reward alignment metric, we need to quantify how well a reward function reflects the preferences of a human stakeholder. To achieve this, we propose the \emph{Trajectory Alignment Coefficient}, a measure based on Kendall's Tau--b correlation.
Kendall's Tau--b is a non-parametric measure that quantifies the level of agreement between two sets of ranked data, adjusting for ties \citep{kendall_tau_b_variant}.
It outputs a scalar value $\in [-1, 1]$, indicating levels of agreement: $1$ for perfect agreement (e.g., identical preference orderings) and $-1$ for complete disagreement (e.g., reverse preference orderings).
 The Trajectory Alignment Coefficient measures the similarity among preference orderings over trajectory distributions in \( H(\mu) \). However, \( H(\mu) \) can theoretically contain an intractably large number of trajectory distributions. To apply the Trajectory Alignment Coefficient as a practical reward alignment measure, we must consider finite subsets of trajectory distributions from $H(\mu)$. 

To compute $\sigma_{\textit{TAC}}$, we first construct two preference data sets, one from a human \( (D_h) \), assumed to be transitive, and one induced by a given $(r, \gamma)$ pair.  
Specifically, we define:
\[
\upsilon_h(D_h) = \Bigl\{ \{\eta_i, \eta_j\} \mid (\eta_i \underset{(h)}\diamond \eta_j) \in D_h\Bigr\}
\]
where \( \diamond \in \{\succ, \prec, \sim\} \) denotes a preference relation and the subscript indicates whether the preference originates from the human or \( (r, \gamma) \).
$\upsilon$ extracts unordered pairs of trajectory distributions that were ranked by the human. Then given these pairs, we construct the corresponding preference dataset under \( (r, \gamma) \) via Definition \eqref{def:preferences_trajectory_distribution}, which ranks trajectory distributions with respect to \( (r, \gamma) \):
\[
D_{r, \gamma}(\upsilon_h, r, \gamma) = \Bigl\{ (\eta_i \underset{(r,\gamma)}\diamond \eta_j) \mid \{\eta_i, \eta_j\} \in \upsilon_h(D_h) \Bigr\}.
\]
Once we have both \( D_h \) and \( D_{r, \gamma} \), \( \sigma_{\textit{TAC}} \) measures their agreement using Kendall’s Tau-b:
\begin{equation}
   \sigma_{\textit{TAC}}(D_{\textit{h}}, D_{r, \gamma}) \doteq \frac{P - Q}{\sqrt{(P + Q + X_0)(P + Q + Y_0)}}
    \label{eq:tac}
\end{equation}
where
\begin{align*}
    P & : \text{Number of concordant pairs between } D_{r, \gamma} \text{ and } D_{\textit{h}}, \\
    Q & : \text{Number of discordant pairs between } D_{r, \gamma} \text{ and } D_{\textit{h}}, \\
    X_0 & : \text{Number of pairs tied only in } D_{r, \gamma}, \\
    Y_0 & : \text{Number of pairs tied only in } D_{\textit{h}}.
\end{align*}

This formulation ensures that \( \sigma_{\textit{TAC}} \) quantifies the alignment between human and reward-induced preferences over the same set of trajectory distribution pairs.
To illustrate, consider the trajectory distributions, $\eta_1, \eta_2, \eta_3$, where $D_{\textit{h}}, D_{r, \gamma}$ are as follows, 
\begin{equation*}
    \begin{gathered}
        D_{\textit{h}} = \Bigl\{ (\eta_2 \underset{(\textit{h})}\succ \eta_3), (\eta_1 \underset{(\textit{h})}\succ \eta_3) \Bigr\} 
        \implies \upsilon_h(D_{\textit{h}}) = \Bigl\{ \{\eta_2, \eta_3\}, \{\eta_1, \eta_3\} \Bigr\}
    \\
        D_{r, \gamma}(\upsilon_h, r, \gamma) = \Bigl\{(\eta_2 \underset{(r, \gamma)}\prec \eta_3), (\eta_1 \underset{(r, \gamma)}\succ \eta_3) \Bigr\}.
    \end{gathered}
\end{equation*}
Note that given a subset of trajectory distributions, the Trajectory Alignment Coefficient can be applied in cases with either a full or partial ranking. A full ranking establishes a complete order over the elements in a subset, where all necessary pairwise comparisons are available. In contrast, a partial ranking occurs when some pairwise comparisons are missing (e.g., $D_{r, \gamma}, D_\textit{h}$, as the comparison between \( \eta_1 \) and \( \eta_2 \) is missing).
This flexibility allows the Trajectory Alignment Coefficient to be used in settings where ranking information is limited or incomplete.

Moreover, the definition for \(\sigma_{\textit{TAC}}\) in Equation \eqref{eq:tac} can also be used to evaluate the differences between any two reward, discount factor pairs: $(r, \gamma), (r', \gamma)$. For example, given data sets \(D_{r, \gamma}\) and \(D_{r', \gamma}\) representing preferences over trajectory distributions, we can use Equation \eqref{eq:tac} to determine how similar the preferences induced by \((r, \gamma)\) are to those induced by \((r', \gamma)\).

\subsection{Invariance to Common Reward Transformations}\label{sec:PotenitalBasedRewardShaping}
Reward shaping is commonly used to accelerate RL training by modifying the reward function to improve learning efficiency \citep{potenital_based_reward_shaping}. Common reward transformations include potential-based reward shaping and positive linear rescaling. A well-designed reward evaluation metric should be invariant to these transformations; otherwise, it may assign different scores to functionally equivalent rewards, leading to inconsistent assessments. We show that the Trajectory Alignment Coefficient maintains this invariance, ensuring stable evaluations of reward alignment.

\begin{definition}\label{def:potenital_based_reward_shaping}

A potential-based reward function is defined as 
$r'(s, a, s') \doteq r(s,a,s') + \gamma \Phi(s') - \Phi(s),
$
given a potential function $\Phi: \mathcal{S} \to \mathbb{R}$, and $\gamma$ as the MDP discount factor.
\end{definition}
To determine whether the Trajectory Alignment Coefficient is invariant to potential-based reward shaping, we first examine the conditions for which preference orderings remain unchanged. While potential-based reward shaping is known to preserve the optimal policy \citep{potenital_based_reward_shaping}, we further prove in Theorem \ref{potenital_based_reward_shaping_both_cases} that, in the infinite-horizon setting, it preserves preference orderings over all trajectory distributions \( \eta \in H(\mu) \) \emph{if and only if they share the same start-state distribution, \( \mu \)}. This establishes a fundamental condition for ensuring that any reward alignment metric based on preference orderings remains unaffected by potential-based reward shaping.

\begin{lemma}\label{linear_transformations_expected_returns_preserves_preferences_main}
Given the infinite-horizon setting, if the expected returns under reward function \( r' \) are a positive linear transformation of the expected returns under reward function \( r \), with respect to all trajectory distributions, then the preference ordering over any two trajectory distributions \( \eta_i \) and \( \eta_j \) remains unchanged. Formally:
\[
\mathbb{E}_{\tau \sim \eta}[G_{r'}(\tau)] = \alpha \mathbb{E}_{\tau \sim \eta}[G_{r}(\tau)] + \beta \implies 
\Big(\eta_i \underset{(r, \gamma)}{\succsim} \eta_j \iff \eta_i \underset{(r', \gamma)}{\succsim} \eta_j\Big) \quad \forall \eta_i, \eta_j,
\]
where \( \alpha > 0 \) and \( \beta \) are constants and the expectations \( \mathbb{E}_{\tau \sim \eta}[G_{r}(\tau)] \) and \( \mathbb{E}_{\tau \sim \eta}[G_{r'}(\tau)] \) are taken over the same trajectory distributions.
\end{lemma}

\begin{lemma}\label{Sufficiency_lemma}[Sufficiency]
    In the infinite horizon setting, if two trajectory distributions \( \eta_i, \eta_j \in H(\mu) \), then potential-based reward shaping preserves their preference ordering with respect to the reward function, $r$, and the potential-based function, $r'$:
    \[
    \eta_i \underset{(r, \gamma)}{\succsim}\eta_j \iff \eta_i \underset{(r', \gamma)}{\succsim} \eta_j.
    \]
    \end{lemma}


\begin{proof}\label{Sufficiency_lemma_proof}
Let \( \eta_i, \eta_j \in H(\mu)\) be arbitrary trajectory distributions, and without loss of generality assume that \( \eta_i \underset{(r, \gamma)}{\succsim} \eta_j \). From Definition \eqref{def:preferences_trajectory_distribution}, this implies that $\mathbb{E}_{\tau \sim \eta_i} [G_r(\tau)] \ge \mathbb{E}_{\tau \sim \eta_j} [G_r(\tau)].$
We now analyze how the expected return changes under the potential-based reward function \( r' \).
The expected return under the reward function \( r \) and the shaped reward function \( r' \) is:
\begin{align}
\mathbb{E}_{\tau \sim \eta}[G_r(\tau)] &= \mathbb{E}_{\tau \sim \eta} \left[ \sum_{t=0}^\infty \gamma^t r(s_t, a_t, s_{t+1}) \right], \label{eq:expected_return_r}\\
\mathbb{E}_{\tau \sim \eta}[G_{r'}(\tau)] &= \mathbb{E}_{\tau \sim \eta} \left[ \sum_{t=0}^\infty \gamma^t r'(s_t, a_t, s_{t+1}) \right]\label{eq:expected_return_r_prime}
\end{align}

Substitute the definition of the potential-based reward function, Definition \eqref{def:potenital_based_reward_shaping}, into Equation \eqref{eq:expected_return_r_prime}:
\[
\mathbb{E}_{\tau \sim \eta}[G_{r'}(\tau)] = \mathbb{E}_{\tau \sim \eta} \left[ \sum_{t=0}^\infty \gamma^t \Big( r(s_t, a_t, s_{t+1}) + \gamma \Phi(s_{t+1}) - \Phi(s_t) \Big) \right]
\]
Distribute $\gamma^t$ and rearrange terms:
\[
\begin{aligned}
\mathbb{E}_{\tau\sim\eta}[G_{r'}(\tau)] &= \mathbb{E}_{\tau\sim \eta} \left[ \sum_{t=0}^\infty \gamma^t  r(s_t, a_t, s_{t+1}) + \gamma^t \gamma\Phi(s_{t+1}) - \gamma^t \Phi(s_t)  \right] \\
    &= \mathbb{E}_{\tau \sim \eta} \left[ \sum_{t=0}^\infty \gamma^t r(s_t, a_t, s_{t+1}) \right] 
    + \mathbb{E}_{\tau \sim \eta} \left[ \sum_{t=0}^\infty \gamma^{t+1} \Phi(s_{t+1}) - \gamma^t \Phi(s_t) \right] \\
    &= \mathbb{E}_{\tau \sim \eta}[G_{r}(\tau)] + \mathbb{E}_{\tau \sim \eta} \left[ \sum_{t=1}^\infty \gamma^{t} \Phi(s_t) - \sum_{t=0}^\infty \gamma^t \Phi(s_t) \right]
\end{aligned}
\]
Split $\sum_{t=0}^\infty \gamma^t \Phi(s_t)$ into two parts, one from $t=1$ to $\infty$ and the other isolating $t=0$:
\[
\begin{aligned}
     \mathbb{E}_{\tau\sim\eta}[G_{r'}(\tau)] &=\mathbb{E}_{\tau \sim \eta}[G_{r}(\tau)] + \mathbb{E}_{\tau \sim \eta} \left[ \sum_{t=1}^\infty \gamma^{t} \Phi(s_t) - \sum_{t=1}^\infty \gamma^t \Phi(s_t) - \gamma^0 \Phi(s_0) \right] 
\end{aligned}
\]
Now, combine like-terms and $\sum_{t=1}^\infty \gamma^{t} \Phi(s_t)$ gets canceled out:
\[
\mathbb{E}_{\tau \sim \eta}[G_{r'}(\tau)] = \mathbb{E}_{\tau \sim \eta}[G_r(\tau)] - \mathbb{E}_{\tau \sim \eta} \left[ \Phi(s_0) \right]
\]
As \( \Phi(s_0) \) depends only on the start-state distribution \( \mu \), and \( \mu \) is the same for all \( \eta \in H(\mu) \), we conclude that the expected returns under \( r \) and \( r' \) differ by a constant,
$\mathbb{E}_{s_{0}\sim \mu} \left[\Phi(s_0) \right]$:
\begin{equation}
\mathbb{E}_{\tau \sim \eta}[G_{r'}(\tau)] = \mathbb{E}_{\tau \sim \eta}[G_r(\tau)] - \mathbb{E}_{s_0 \sim \mu} \left[ \Phi(s_0) \right] \label{eq:shaping_expected_return_with_phi}
\end{equation}
As $\mathbb{E}_{\tau \sim \eta}[G_{r'}(\tau)]$ is a positive linear transformation of $\mathbb{E}_{\tau \sim \eta}[G_{r}(\tau)]$, we apply Lemma \ref{linear_transformations_expected_returns_preserves_preferences_main} and conclude that the preference remains unchanged under reward shaping:
\[
\eta_i \underset{(r, \gamma)}{\succsim} \eta_j \iff \eta_i \underset{(r', \gamma)}{\succsim} \eta_j \; \forall \eta_i, \eta_j \in H(\mu)
\]
\end{proof}
\begin{lemma}\label{Necessity_lemma_main}[Necessity]
   In the infinite horizon setting, if two trajectory distributions \( \eta_i \in H(\mu_i) \) and \( \eta_j \in H(\mu_j) \) have different start-state distributions (\(\mu_i \neq \mu_j\)), then there exists a potential function \( \Phi \) such that:
    \[
    \eta_i \underset{(r, \gamma)}{\succsim}  \eta_j \textit{ and } \eta_i \underset{(r', \gamma)}{\prec} \eta_j.
    \]
    \end{lemma}

We leave the proof for the necessity condition (Lemma \eqref{Necessity_lemma_main} to Supplementary Material \ref{sec:proofs_detailed}.
\begin{theorem}\label{potenital_based_reward_shaping_both_cases}
In the infinite-horizon setting, let $(r,\gamma)$ and two trajectory distributions \( \eta_i \in H(\mu_i) \) and \( \eta_j \in H(\mu_j) \) be given. Potential-based reward shaping is guaranteed to maintain the preference ordering over all trajectory distributions if and only if \( \mu_i = \mu_j \).
Formally,  
\begin{gather*}
    \big(\eta_i \underset{(r, \gamma)}{\succsim} \eta_j \iff \eta_i \underset{(r', \gamma)}{\succsim} \eta_j\big)
     \; \forall \Phi\in\mathcal{F} \iff \mu_i = \mu_j
\end{gather*}
\end{theorem}
where $\mathcal{F}$ is the space of all potential-based shaping functions and $\Phi:\mathcal{S}\mapsto\mathbb{R}$ is an arbitrary function in this space. The proof of Theorem \eqref{potenital_based_reward_shaping_both_cases} follows directly from Lemmas \eqref{Sufficiency_lemma} and \eqref{Necessity_lemma_main}.

\begin{definition}\label{tac_invariance}  
Let \( D_{r, \gamma} \), \( D_{r', \gamma} \), and \( D_h \) be preference data sets over trajectory distributions induced by \( (r, \gamma) \), \( (r', \gamma) \), and a human, respectively, where \( r' \doteq f(r) \) for some transformation \( f \).  
The Trajectory Alignment Coefficient \( \sigma_{\textit{TAC}} \) is invariant to \( f \) if and only if $
\sigma_{\text{TAC}}(D_h, D_{r, \gamma}) = \sigma_{\text{TAC}}(D_h, D_{r', \gamma})$.
\end{definition}

 \begin{theorem}\label{tac_invariance_pbrf}
Consider the infinite-horizon setting. 
Let $r$ and $r'$ be reward functions where \( r' \) is a shaped version of \( r \) using potential-based reward shaping. 
Let $D_h$ be a data set of human preferences over trajectory distributions, and define $D_{r, \gamma} = D_{r, \gamma}(\upsilon_h, r, \gamma)$ and $D_{r', \gamma} = D_{r', \gamma}(\upsilon_h, r', \gamma)$ as the preference data sets induced by $(r, \gamma)$ and $(r', \gamma)$, respectively. The Trajectory Alignment Coefficient is invariant to any potential-based reward shaping if and only if within each set of trajectory distributions $\{\eta_i, \eta_j\} \in \upsilon_h(D_h)$, $\eta_i, \eta_j$ share the same initial state distribution 
(i.e., $\eta_i, \eta_j \in$ \(H(\mu)\)).
\end{theorem}

\begin{proof}
Let $\{\eta_i, \eta_j\} \in \upsilon_h(D_h)$ be an arbitrary pair of trajectory distributions compared in the human preference data set. 
By Definition \eqref{eq:tac} of $\sigma_{\text{TAC}}$, the following biconditional holds:
\begin{gather*}
\sigma_{\text{TAC}}(D_h, D_{r, \gamma}) = \sigma_{\text{TAC}}(D_h, D_{r', \gamma}) 
\iff \forall \{\eta_i, \eta_j\} \in \upsilon_h(D_h): \big(\eta_i \underset{(r, \gamma)}\diamond \eta_j \iff \eta_i \underset{(r', \gamma)}\diamond \eta_j\big)
\end{gather*}
where $\diamond \in \{\succ, \prec, \sim\}$ denotes the preference relation.
By Theorem \eqref{potenital_based_reward_shaping_both_cases}, we have:
\begin{gather*}
    \big(\eta_i \underset{(r, \gamma)}{\diamond} \eta_j \iff \eta_i \underset{(r', \gamma)}{\diamond} \eta_j\big)
     \; \forall \Phi\in\mathcal{F}, \eta_i \in H(\mu_i), \eta_j \in H(\mu_j) \iff \mu_i = \mu_j
\end{gather*}
Therefore, we conclude that $\sigma_{\textit{TAC}}$ is invariant to potential-based reward shaping if and only if all sets of trajectory distributions being compared share the same initial state distribution, $\mu$:
\begin{gather*}
\sigma_{\text{TAC}}(D_h, r, \gamma) = \sigma_{\text{TAC}}(D_h, r', \gamma)\;\forall \Phi\in\mathcal{F} \iff \eta_i, \eta_j \in H(\mu), \; \forall \{\eta_i, \eta_j\} \in \upsilon_h(D_h)
\end{gather*}
\end{proof}
\begin{theorem}\label{tac_invariant_pos_linear_transformation}
 Given the infinite-horizon setting, the Trajectory Alignment Coefficient is invariant to positive linear transformations. 
\end{theorem}
To prove Theorem \eqref{tac_invariant_pos_linear_transformation}, we show that a positive linear transformation linearly transforms the expected return. We can then apply Lemma \eqref{linear_transformations_expected_returns_preserves_preferences_main} and the Trajectory Alignment Coefficient's invariance definition \eqref{tac_invariance} to complete the proof. The full derivation is provided in Supplementary Material~\ref{sec:proofs_detailed}.

\subsection{Trajectory Alignment Coefficient in Practice}\label{sec:tac_in_practice_main}
To use the Trajectory Alignment Coefficient in practice, two parameters must be specified: the number of trajectories (or trajectory distributions) to rank and the trajectory sampling method. To obtain a meaningful evaluation, it is important to include a diverse set of trajectories. If trajectories are limited to a specific region of the state and action space, the evaluation may fail to reveal reward misalignment in other regions.
To mitigate this, one heuristic we propose is to sample trajectories that exhibit qualitatively different behaviors. In our experiments, we generate these trajectories via RL agents partially trained with different reward functions. However, other sources, such as human demonstrations or trajectories generated through behavioral cloning, could also be used. 
Beyond diversity, the number of ranked trajectories plays a critical role. While more trajectories provide a clearer picture of reward alignment, ranking too many is impractical for humans. We found that in our tested domain (Hungry-Thirsty, described in Section \ref{sec:experimental_design}), the Trajectory Alignment Coefficients from a subset of $12$ trajectories were highly correlated 
with those from a set of $1200$, suggesting that a smaller, well-chosen set can still provide reliable estimates. See Supplementary Material \ref{sec:traj_sample_size_study} for further details.

To illustrate the Trajectory Alignment Coefficient in practice, we revisit the toy example in Figure \ref{fig:AD_alignnment_example}, demonstrating how reward-based rankings can diverge from human preferences
Specifically, the toy $(r, \gamma)$ pair produced a preference ordering of
$\tau_{\text{{success}}} \succ \eta_{\text{success-crash}} \succ \tau_{\text{{idle}}} \succ \tau_{\text{{crash}}}$. However, a human stakeholder would likely prefer remaining parked over possibly crashing: 
$\tau_{\text{{success}}} \succ \tau_{\text{{idle}}} \succ \eta_{\text{success-crash}} \succ \tau_{{{crash}}}$.
To compute \(\sigma_{\textit{TAC}}\) from Equation \eqref{eq:tac}, we count the number of concordant and discordant pairs. In our four-element example, six pairwise comparisons are possible, with all but one being concordant. The only discordant pair is \( (\eta_{\text{success-crash}}, \tau_{\text{idle}})\), where the preference is reversed. Since there are no ties, \(X_0\) and \(Y_0\) are zero. Substituting these values into the equation yields  \(\sigma_{\textit{TAC}}\approx 0.67\), indicating misalignment between \((r, \gamma)\) and the human stakeholder's preferences.

\section{Experimental Design}\label{sec:experimental_design}

This study examines the reward design setting where RL practitioners have to choose between reward functions in order to satisfy the preferences of another stakeholder (e.g., a domain expert). In particular, \emph{our goal is to investigate whether the Trajectory Alignment Coefficient can assist RL practitioners in reward selection}.
We assess this by comparing RL practitioners with and without access to our metric, focusing on two key dimensions:
\begin{enumerate}
    \item \textbf{Perceived Benefit} —  Does it reduce perceived cognitive workload, increase ease of use, and improve understanding of reward functions?
    \item \textbf{Practical Impact} — Does access to the metric help RL practitioners choose reward functions that improve performance of learned policies while also reducing the time spent on reward selection?
    
\end{enumerate}

We conducted an ethics-approved human subject study with $11$ self-identified RL practitioners. This included individuals who had completed graduate courses in RL $(81\text{\%})$, conducted RL research $(100\text{\%})$, or applied RL in their professional work $(27\text{\%})$. Note these categories were not mutually exclusive. In this within-subjects study, participants selected reward functions under three experimental conditions with different types of assistance. 
The study was primarily in-person and each session lasted approximately $90$ minutes.

\paragraph{Testbed: Hungry-Thirsty}
Our study is conducted in a modified Hungry-Thirsty domain \citep{singh2009rewards}, an environment where others have shown~\citep{perils_reward_design} that RL practitioners struggle with reward design. It is a $4 \times 4$ grid-world where food and water are randomly placed at the grid corners (see Figure \ref{fig:hungry_thirsty_pic} in Supplementary Material \ref{sec:env_details}). The agent can move in one of the four cardinal directions or execute eat or drink actions. The agent's goal is to maximize time spent without hunger. Hunger occurs if the agent has not eaten in the previous timestep, but eating is only possible at a food source when the agent is not thirsty. If the agent is thirsty, the eat action fails. The agent becomes thirsty with $0.10$ probability per step. This is an infinite-horizon MDP, although each episode is truncated after $200$ timesteps.
The state space consists of the agent’s position and two Boolean variables for hunger and thirst. The reward function is a linear combination of these variables.  The evaluation metric is the number of timesteps the agent is not hungry.

\paragraph{Study Protocol}
The study consisted of two primary components, Preference Review and Reward Selection. 
In the \emph{Preference Review} component, participants first read a description of the Hungry-Thirsty domain and then completed a short quiz and an interactive game-play session to confirm their understanding of its rules.
They were informed that they would be collaborating with a domain expert to select a reward function for training an RL agent, with the expert providing a ranking of $12$ trajectories (generated using the task evaluation metric as a proxy for expert preferences). 
Participants then reviewed this ranking alongside corresponding video clips.
To obtain these trajectories, we use the mixture sampling method described in Section \ref{sec:tac_in_practice_main}.
\begin{wrapfigure}{r}{0.42\textwidth}  
  \centering
  \includegraphics[width=0.70\textwidth, trim=0cm 10cm 5cm 5cm, clip]{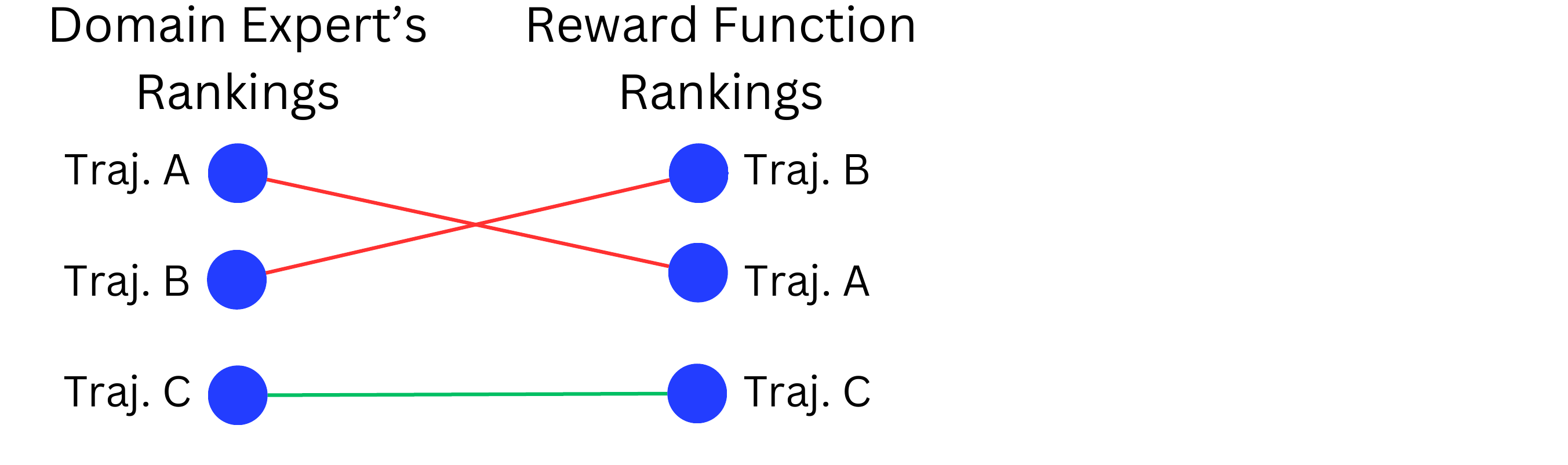}
  \caption{Visualization comparing rankings of $3$ trajectories, used in the Reward $+$ Alignment $+$ Visualization condition.}
  \label{fig:example_visualization}
\end{wrapfigure}

In the \emph{Reward Selection phase}, participants completed four rounds of reward selection, with the goal of choosing the reward function that best reflects the domain expert's preferences.
To select reward functions for this component, we considered those from the open-sourced human reward data set in \cite{perils_reward_design} and their affine transformations. See Table \ref{tab:final_return_comparison} in Supplementary Material \ref{sec:ui} for the complete set of reward function comparisons. 
When selecting reward functions for comparisons, we focused on reward functions that: (1) differed in magnitude, scale, or range, and (2) 
myopically ranked states based on immediate reward. Across all conditions, participants had access to two reward functions at a time and could revisit the trajectory rankings along with their respective video clips. Moreover, reward function pairs were randomized and were not repeated across the study.
See Figures \ref{fig:user_study_interface1}--\ref{fig:visualization} in Supplementary Material \ref{sec:ui} for the user interface.
The study contained three conditions, whose order was randomized for each participant:
\begin{itemize}
    \item \textbf{Reward Only (Control):} The reward functions were shown but no further information was given. 
    \item \textbf{Reward + Alignment:} Participants also received the Trajectory Alignment Coefficient, computed from the domain expert's preferences and those induced by each reward function.
    \item \textbf{Reward + Alignment + Visual:} 
Participants were also provided the Trajectory Alignment Coefficient and a parallel coordinate plot illustrating differences in the domain expert's and reward functions' rankings over trajectories (see example in Figure \ref{fig:example_visualization}).
\end{itemize}
\paragraph{Evaluation}
To assess differences in \emph{perceived benefit} across conditions, participants completed a modified NASA Task Load Index (NASA-TLX) \citep{hart1988development}, which measured cognitive work load (scale from $1$--$7$). The survey included all six NASA-TLX questions, along with three additional questions on confidence in reward selection, helpfulness of feedback (e.g., alignment, visual, and/or reward function), and ease of integrating feedback into decisions (see Figure \ref{fig:condition_experiences_survey} in Supplementary Material \ref{sec:ui}). 
Participants completed this survey after each condition.
We then compute the overall workload based on the survey responses.
After completing all reward selection conditions, participants voted on which condition best improved their understanding of the reward function, provided the most useful feedback, was least mentally demanding, and made reward selection the easiest (see Figure \ref{fig:comparison_survey} in Supplementary Material \ref{sec:ui}).
We also included open-ended questions for participants to describe their experiences during the reward selection process.

To assess the \emph{practical impact} of the conditions, we first examined how often users selected reward functions that improved policy performance compared to the unselected alternative rewards. 
For shorthand, we refer to these reward functions as performant or policy-improving. 
Specifically, we calculated the proportion of times users chose the reward function that resulted in a higher final return and a greater area under the learning curve (AUC), with respect to the evaluation metric.
To evaluate the policy performance of the reward functions, we trained Q--Learning, SARSA, and Expected SARSA agents on each reward function, performing a grid search over learning rate $\in \{10^{-k}, 5 \times 10^{-k} \mid k \in \{2, 3, 4\}\}$, and epsilon $\in \{0.01, 0.05, 0.15\}$. Each agent was trained across $10$ environment seeds. We then averaged final returns and AUC, separately,  across all trained agents to assess performance. By using multiple RL algorithms and varying hyperparameters, we aimed to reduce the likelihood that a reward function’s performance was due to random chance or a particularly favorable choice of hyperparameters.
Second, we measured the time taken to complete the reward selection process.


For all analyses, we used paired \(t\)--tests for continuous data when normality assumptions held and Wilcoxon Signed-Rank tests otherwise. For categorical voting data, Fisher’s Exact Test was applied.
The corresponding 
$p$--values and test statistics are reported: \(t\) (paired \(t\)--test) and \(W\) (Wilcoxon Signed-Rank test).
We performed the Bonferroni correction (\(\alpha=0.05\)) to control for Type~I errors.
\section{Results}
This section presents the results from the user study, structured around our two research components, perceived benefit and practical impact, as well as a qualitative analysis of participants' experiences.
\paragraph{Perceived Benefit}
Figure \ref{fig:survey_results}(a) presents the results of the NASA-TLX survey: the average ratings ($\pm$ standard error) for five selected questions and the overall workload score. We chose a subset of questions for this plot to avoid redundancy.
We found that the overall workload score was significantly lower for the Reward + Alignment ($\mu=1.97$, $p=0.003$, $t=-3.46$) and Reward + Alignment + Visual ($\mu=1.83$, $p=0.02$, $W=10.0$) conditions compared to the Reward Only condition ($\mu=3.19$), representing a $1.5x$ reduction in workload. 
Similarly, \ref{fig:survey_results}(b) shows the total number of participant votes for each condition, reflecting preferences across different criteria.
Notably, $100$\% of participants reported that either alignment conditions led to easier decision-making. Additionally, $91$\% indicated that either alignment conditions improved their understanding of the reward functions and reduced mental demand, while $82$\% found the provided information most useful---all of which are significantly greater than the number of votes for the Reward Only condition $(p \le 0.009)$.


\paragraph{Practical Impact}

In Figure \ref{fig:practical_impact_results}(b), we found that participants achieved significantly greater success in selecting the policy-improving reward functions in both the Reward + Alignment ($\mu=0.93$, $p=0.01$, $W=41.0$) and Reward + Alignment + Visual ($\mu=0.96$, $p=0.008$, $W=28.0$) conditions compared to the Reward Only condition ($\mu=0.66$). Specifically, we found that in the Reward Only condition, $55$\% of participants selected policy-improving reward functions no better than random (or worse). 
However, the time-to-completion data, shown in Figure \ref{fig:practical_impact_results}(a), provide a more complete picture. While participants on average took longer in the Reward Only condition ($\mu = 660.89$) compared to both the Reward + Alignment ($\mu = 334.09$, $p = 0.04$, $W = 13.0$) and Reward + Alignment + Visual conditions ($\mu = 393.44$, $p = 0.07$, $W = 16.0$), these differences were not statistical significant. It is important to acknowledge that deriving the preference ordering used in the alignment conditions itself requires time, which is not accounted for in this comparison.

Beyond aggregate trends, individual differences in time use revealed interesting patterns. Notably, six participants (P$6$--P$11$) spent more time in the Reward Only condition but still performed worse in reward selection, highlighting that more time without alignment support did not lead to better outcomes.
Moreover, three participants (P$1$--P$3$) achieved perfect success rates in the Reward Only condition while spending less time than in the Reward + Alignment + Visual condition. 
This suggests that they may have required less assistance during reward selection, and the additional visual feedback in the Reward + Alignment + Visual condition likely introduced more information for them to process, increasing deliberation time. 
These results suggest that while alignment-based feedback improved reward selection success for most participants, some succeeded without it.
\begin{figure}[h!]
  \centering

  \hfill
  {\includegraphics[width=0.99\textwidth, trim = 0cm 4.3cm 0cm 9cm, clip]{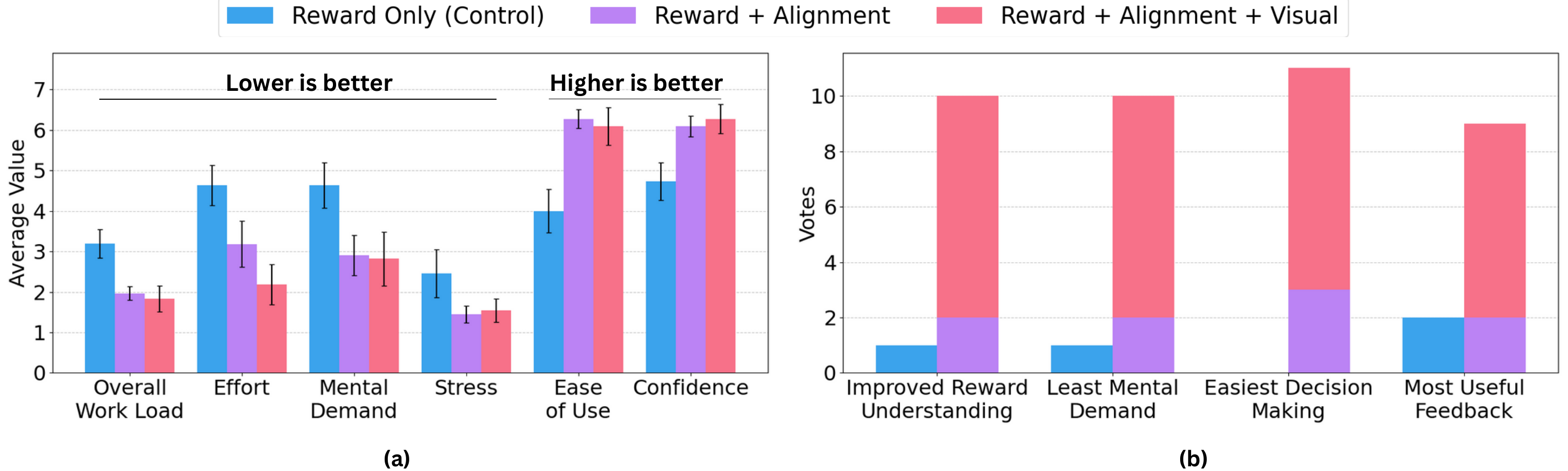}}

  \caption{Results from participants' experience during reward selection are shown from (a) participants given the NASA-TLX survey and (b) a survey assessing different aspects of favorability.} 
  
  \label{fig:survey_results}
\end{figure}
\begin{figure}[h!]
  \centering

  \hfill
  {\includegraphics[width=0.99\textwidth, trim = 0cm 4.5cm 0cm 3.3cm, clip]{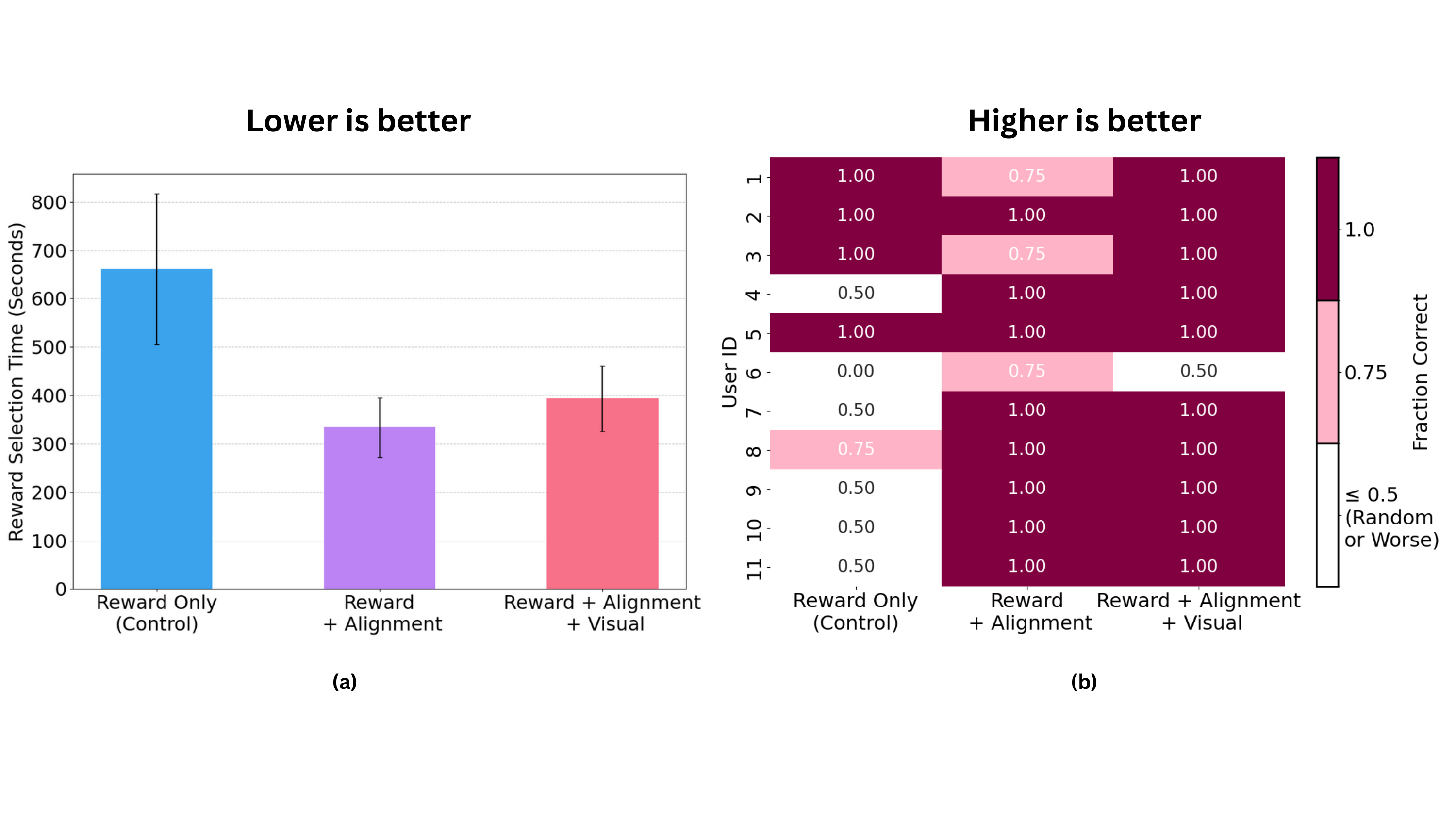}}
   \caption{Results showing (a) the mean completion time ($\pm$ standard error) for reward selection and (b) the proportion of policy-improving reward functions selected per user and condition.} 
  \label{fig:practical_impact_results}
\end{figure}

\paragraph{Qualitative Analysis}
In the open-ended questions, we asked participants to explain which condition they liked and disliked. Most participants favored the Reward + Alignment + Visual condition ($73$\%). A common theme among these participants was the emphasis on how the combination of the visualization and alignment score provided both intuitive insights into the reward function’s behavior and a scalar metric that simplified decision-making.
For example, \textbf{P5} stated ``It also let me see exactly which trajectories were aligned vs not, giving me better insights into what behavior the reward function was favoring.''
Furthermore, the Reward Only condition was least favored by $64$\% of participants. Two main complaints emerged: (1) difficulty interpreting the reward functions, requiring intuition or domain expertise, and (2) the process being slow and tedious, making reward comparisons time-consuming.
Specifically, \textbf{P2} wrote ``it is very hard to just look at the reward function and guess what will happen,'' while \textbf{P8} mentioned ``Even though all the information needed to make a decision could be deducted from the trajectories, going through all of them on a piece of paper could be very hard and time-consuming.''  
Lastly, two participants least preferred the Alignment + Reward condition. The common theme was that the alignment score alone was insufficient. Unsurprisingly, participants felt that, compared to the Reward + Alignment + Visual condition, it lacked detailed feedback and appeared too aggregated without supporting visuals, making decision-making more difficult. Overall, the open-ended responses further supported the quantitative evidence, indicating that the Trajectory Alignment Coefficient improved the user experience during reward selection.

\section{Conclusion}

The success of RL agents is inherently dependent on the quality of the  MDP's reward function, yet reward design is often treated as a secondary concern. 
In practice, however, it is a complex and error-prone process. These challenges are further amplified in real-world RL applications, where reward design is typically a collaborative effort between RL practitioners and domain experts. 
 In such settings, the RL practitioner must design a reward function that accurately reflects the domain expert’s preferences and constraints.
In this work, we address this challenge by introducing the Trajectory Alignment Coefficient, a reward alignment metric that quantifies the similarity between a human stakeholder's preference orderings over trajectory distributions and those induced by a reward function. Through an $11$--person user study, we demonstrate its effectiveness in supporting RL practitioners during reward selection. Specifically, participants in Trajectory Alignment-based conditions reported significantly lower cognitive workload and were more likely to select policy-improving reward functions. 
In future work, we plan to extend our metric to full reward design, where participants construct reward functions from scratch.

\subsubsection*{Broader Impact Statement}
\label{sec:broaderImpact}
In this work, we introduce a reward evaluation metric that measures the alignment between a human stakeholder's preferences over trajectory distributions and those induced by a reward function. However, if the preferences provided by the human stakeholder do not accurately reflect their true beliefs, the output of the metric may be unreliable and could mislead the reward design process. 
\subsubsection*{Acknowledgments}\label{sec:ack}
This work has taken place in part in the Intelligent Robot Learning Lab at the University of Alberta and the Rewarding Lab at the University of Texas at Austin. The Intelligent Robot Learning Lab is supported by research grants from Alberta Innovates; the Alberta Machine Intelligence Institute (Amii); a Canada CIFAR AI Chair, Amii; the Digital Research Alliance of Canada; Mitacs; and the Natural Sciences and Engineering Research Council (NSERC). The Rewarding Lab is supported by the National Science Foundation (NSF IIS-2402650), the Office of Naval Research (ONR N00014-22-1-2204), EA Ventures, Bosch, UT Austin’s Good Systems grand challenge, and Open Philanthropy.



\bibliography{main}

\begin{thebibliography}{18}
\providecommand{\natexlab}[1]{#1}
\providecommand{\url}[1]{\texttt{#1}}
\expandafter\ifx\csname urlstyle\endcsname\relax
  \providecommand{\doi}[1]{DOI: #1}\else
  \providecommand{\doi}{DOI: \begingroup \urlstyle{rm}\Url}\fi

\bibitem[Amodei et~al.(2016)Amodei, Olah, Steinhardt, Christiano, Schulman, and Mané]{amodei2016concreteproblemsaisafety}
Dario Amodei, Chris Olah, Jacob Steinhardt, Paul Christiano, John Schulman, and Dan Mané.
\newblock {Concrete Problems in AI Safety}.
\newblock \emph{CoRR, abs/1606.06565}, 2016.

\bibitem[Booth et~al.(2023)Booth, Knox, Shah, Niekum, Stone, and Allievi]{perils_reward_design}
Serena Booth, W.~Bradley Knox, Julie Shah, Scott Niekum, Peter Stone, and Alessandro Allievi.
\newblock {The Perils of Trial-and-Error Reward Design: Misdesign through Overfitting and Invalid Task Specifications}.
\newblock In \emph{AAAI Conference on Artificial Intelligence}, 2023.

\bibitem[Bowling et~al.(2023)Bowling, Martin, Abel, and Dabney]{bowling2023settling}
Michael Bowling, John~D. Martin, David Abel, and Will Dabney.
\newblock {Settling the Reward Hypothesis}.
\newblock In \emph{International Conference on Machine Learning}, 2023.

\bibitem[Brown et~al.(2021)Brown, Schneider, and Niekum]{value_alignment}
Daniel~S. Brown, Jordan Schneider, and Scott Niekum.
\newblock {Value Alignment Verification}.
\newblock In \emph{International Conference on Machine Learning}, 2021.

\bibitem[Gleave et~al.(2021)Gleave, Dennis, Legg, Russell, and Leike]{epic}
Adam Gleave, Michael Dennis, Shane Legg, Stuart Russell, and Jan Leike.
\newblock {Quantifying differences in reward functions}.
\newblock In \emph{International Conference on Learning Representations}, 2021.

\bibitem[Hadfield-Menell et~al.(2016)Hadfield-Menell, Russell, Abbeel, and Dragan]{hadfield2016cooperative}
Dylan Hadfield-Menell, Stuart~J Russell, Pieter Abbeel, and Anca Dragan.
\newblock {Cooperative inverse reinforcement learning}.
\newblock In \emph{Neural Information Processing Systems}, 2016.

\bibitem[Hart \& Staveland(1988)Hart and Staveland]{hart1988development}
Sandra~G. Hart and Lowell~E. Staveland.
\newblock {Development of NASA-TLX (Task Load Index): Results of Empirical and Theoretical Research}.
\newblock \emph{Advances in Psychology}, 52:\penalty0 139--183, 1988.

\bibitem[Kendall(1945)]{kendall_tau_b_variant}
Maurice~George Kendall.
\newblock {The Treatment of Ties in Ranking Problems}.
\newblock \emph{Biometrika}, 33\penalty0 (3):\penalty0 239--251, 1945.

\bibitem[Knox \& MacGlashan(2024)Knox and MacGlashan]{specifying_rl_objectives}
W.~Bradley Knox and James MacGlashan.
\newblock {How to Specify Reinforcement Learning Objectives}.
\newblock In \emph{Finding the Frame: An RLC Workshop for Examining Conceptual Frameworks at the Reinforcement Learning Conference}, 2024.

\bibitem[Knox et~al.(2023)Knox, Allievi, Banzhaf, Schmitt, and Stone]{reward_misdesign_AD}
W.~Bradley Knox, Alessandro Allievi, Holger Banzhaf, Felix Schmitt, and Peter Stone.
\newblock {Reward (Mis)design for Autonomous Driving}.
\newblock \emph{Artificial Intelligence}, 316\penalty0 (103829), 2023.

\bibitem[Ng et~al.(1999)Ng, Harada, and Russell]{potenital_based_reward_shaping}
Andrew~Y. Ng, Daishi Harada, and Stuart~J. Russell.
\newblock {Policy Invariance Under Reward Transformations: Theory and Application to Reward Shaping}.
\newblock In \emph{International Conference on Machine Learning}, 1999.

\bibitem[Pan et~al.(2022)Pan, Bhatia, and Steinhardt]{pan2022effectsrewardmisspecificationmapping}
Alexander Pan, Kush Bhatia, and Jacob Steinhardt.
\newblock {The Effects of Reward Misspecification: Mapping and Mitigating Misaligned Models}.
\newblock In \emph{International Conference on Learning Representations}, 2022.

\bibitem[Pignatelli et~al.(2024)Pignatelli, Ferret, Geist, Mesnard, van Hasselt, Pietquin, and Toni]{pignatelli2024surveytemporalcreditassignment}
Eduardo Pignatelli, Johan Ferret, Matthieu Geist, Thomas Mesnard, Hado van Hasselt, Olivier Pietquin, and Laura Toni.
\newblock {A Survey of Temporal Credit Assignment in Deep Reinforcement Learning}.
\newblock \emph{Transactions on Machine Learning Research}, 2024.

\bibitem[Shen et~al.(2023)Shen, Jin, Huang, Liu, Dong, Guo, Wu, Liu, and Xiong]{shen2023largelanguagemodelalignment}
Tianhao Shen, Renren Jin, Yufei Huang, Chuang Liu, Weilong Dong, Zishan Guo, Xinwei Wu, Yan Liu, and Deyi Xiong.
\newblock {Large Language Model Alignment: A Survey}.
\newblock \emph{CoRR, abs/2309.15025}, 2023.

\bibitem[Singh et~al.(2009)Singh, Lewis, and Barto]{singh2009rewards}
Satinder Singh, Richard~L Lewis, and Andrew~G Barto.
\newblock {Where Do Rewards Come From?}
\newblock In \emph{Conference of the Cognitive Science Society}, 2009.

\bibitem[Skalse et~al.(2022)Skalse, Howe, Krasheninnikov, and Krueger]{defining_reward_gaming}
Joar Skalse, Nikolaus Howe, Dmitrii Krasheninnikov, and David Krueger.
\newblock {Defining and Characterizing Reward Gaming}.
\newblock In \emph{Neural Information Processing Systems}, 2022.

\bibitem[Sutton \& Barto(2018)Sutton and Barto]{sutton2018reinforcement}
Richard~S Sutton and Andrew~G Barto.
\newblock \emph{{Reinforcement learning: An Introduction}}.
\newblock MIT press, 2018.

\bibitem[Wulfe et~al.(2022)Wulfe, Balakrishna, Ellis, Mercat, McAllister, and Gaidon]{dard}
Blake Wulfe, Ashwin Balakrishna, Logan Ellis, Jean Mercat, Rowan McAllister, and Adrien Gaidon.
\newblock {Dynamics-aware comparison of learned reward functions}.
\newblock In \emph{International Conference on Learning Representations}, 2022.

\end{thebibliography}
\bibliographystyle{rlj}

\beginSupplementaryMaterials
\appendix


\section{Trajectory Alignment Coefficient in Practice}\label{sec:tac_in_practice_appendix}
\subsection{How to sample trajectories?}
To sample the trajectories used for the trajectory alignment coefficient, we propose that practitioners select qualitatively different trajectories. We now outline the methodology used to obtain these trajectories.

To ensure qualitative diversity, we sampled trajectories from Q--learning agents that were only partially trained. 
Note that for this component, we only considered the environment configuration with a fixed start state of $(0,0)$ and where the food and water locations are $(3,0)$ and $(0,0)$, respectively. We did this intentionally, as the trajectory alignment coefficient only remains invariant to potential-based reward shaping if the start state is fixed.

We used the default hyperparameters from Table \ref{tab:Q_Learning_hyperparams} and the evaluation metric as the reward function for this training. We categorized partial training into three groups: \textbf{low-return}, \textbf{medium-return}, and \textbf{high-return}. After partial training, we performed offline evaluation (i.e., policy rollouts with no exploration). The low-return group contained trajectories with returns in $[1,30)$, the medium-return group had returns in $[30,60)$, and the high-return group had returns $\geq 60$. We then randomly sampled four trajectories per group, resulting in a total of $12$ trajectories used in the user study outlined in Section \ref{sec:experimental_design}.

The specific returns (per the evaluation metric) of the $12$ trajectories are as follows:
\[[1.0,6.0,9.0,29.0,43.0,56.0,56.0,66.0,68.0,74.0,90.0]\]
 Note that the optimal policy achieves an average return of $\approx 96.31$. We computed the optimal policy by performing value iteration over $13$ seeds.

\subsection{How many trajectories to samples?}\label{sec:traj_sample_size_study}

In Section \ref{sec:tac_in_practice_main}, we noted that in Hungry Thirsty, the Trajectory Alignment Coefficients computed from a small subset of $12$ trajectories were highly correlated with those computed from a larger set of $1200$ trajectories. To better understand this relationship, we now outline the methodology used to obtain this result.

More specifically, we used all $31$ reward functions from the open-sourced human reward data set from \cite{perils_reward_design}, along with their variants (e.g., $23$ added reward functions, linear transformations), resulting in a total of $54$ reward functions. We then sampled trajectory subsets of varying sizes (\(N \in \{10, 12, 25, 100, 500\}\)) using the sampling strategy described earlier.
Using the task evaluation metric as a proxy for the domain expert preferences, we calculated the Trajectory Alignment Coefficients ($\sigma_{TAC}$) between the domain expert’s preferences and those induced by the reward functions. 
To assess whether smaller trajectory subsets provide reliable \(\sigma_{TAC}\) estimates, we computed the correlation between the \(\sigma_{TAC}\) scores from each subset to those obtained using a larger trajectory set of $1200$.
We repeated this process $50$ times per trajectory subset size to account for variability and then averaged the resulting correlations. This is depicted in Table \ref{tab:correlation_results}.

A high correlation between the Trajectory Alignment Coefficients from smaller subsets and the $1200$--trajectory set would indicate that even with a limited number of trajectories, we can obtain $\sigma_{TAC}$ estimates that are consistent with those derived from a significantly larger sample. This finding suggests that in Hungry-Thirsty, a relatively small number of trajectories may be sufficient for accurately assessing $\sigma_{TAC}$, reducing the need for extensive trajectory ranking.

We used \textit{Kendall’s Tau} to measure correlation because (1) the normality assumption was violated, making Pearson’s \(r\) unsuitable, and (2) there were ties in the dataset, which makes Kendall’s Tau a better choice than \textit{Spearman’s Rho}, as it handles tied ranks more effectively.
\begin{table}[h!]
\centering
\resizebox{\textwidth}{!}{%
\begin{tabular}{l c c}  
\toprule
\textbf{\textsc{Subset Size}} & \textbf{\textsc{Average Correlation}} & \textbf{\textsc{Standard Deviation}} \\  
\midrule
\textsc{500} & \textsc{0.992} & \textsc{0.005} \\  
\textsc{100} & \textsc{0.979} & \textsc{0.009} \\  
\textsc{25} & \textsc{0.930} & \textsc{0.036} \\  
\textsc{12} & \textsc{0.828} & \textsc{0.105} \\  
\textsc{10} & \textsc{0.795} & \textsc{0.131} \\  
\bottomrule
\end{tabular}
}  
\captionsetup{width=\textwidth}
\caption{Average correlation (across $50$ samples) between the \(\sigma_{TAC}\) scores computed from each subset size and those obtained using a larger trajectory set of $1200$ (denoted in bold). The standard deviation across samples is reported in the third column.}
\label{tab:correlation_results}
\end{table}

\subsection{Q-Learning Hyperparameters}
An overview of the hyperparameters used for training the Q--Learning, SARSA, and Expected SARSA algorithms are provided in Table~\ref{tab:Q_Learning_hyperparams}. To evaluate the performance of the reward functions used in the reward selection aspect of the user study, we trained $18$ Q--Learning, SARSA and Expected SARSA agents by performing a full grid search over two hyperparameters: learning rate and epsilon. We systematically varied both across all combinations while keeping the remaining hyperparameters fixed.

\begin{table}[h!]
\centering
\resizebox{\textwidth}{!}{%
\begin{tabular}{l c}  
\toprule
\textbf{\textsc{Hyperparameter}} & \textbf{\textsc{Value}} \\  
\midrule
\textsc{Number of Training Episodes} & \textsc{10000} \\  
\textsc{Number of Seeds} & \textsc{10} \\  
\textsc{Learning Rate} & \textsc{[0.0001, 0.001, 0.01, 0.0005, 0.005, \textbf{0.05}]} \\  
\textsc{Exploration Strategy} & \textsc{Epsilon-Greedy} \\  
\textsc{Epsilon} & \textsc{[0.05, 0.10,\textbf{0.15}]} \\  
\textsc{Discount} & \textsc{0.99} \\  
\bottomrule
\end{tabular}
}  
\captionsetup{width=\textwidth}
\caption{Hyperparameters for all RL Algorithms. For hyperparameters with multiple options, the list represents possible values that were searched over. Bolded values indicate default settings. }
\label{tab:Q_Learning_hyperparams}
\end{table}

\section{User Study}\label{sec:ui}
We first show figures that correspond to the interface used in the human-subject study. 
\begin{figure}[H]
  \centering
  \begin{subfigure}{0.48\textwidth}
    \centering
    \includegraphics[scale=0.33, trim=2cm 1cm 2cm 1cm, clip]{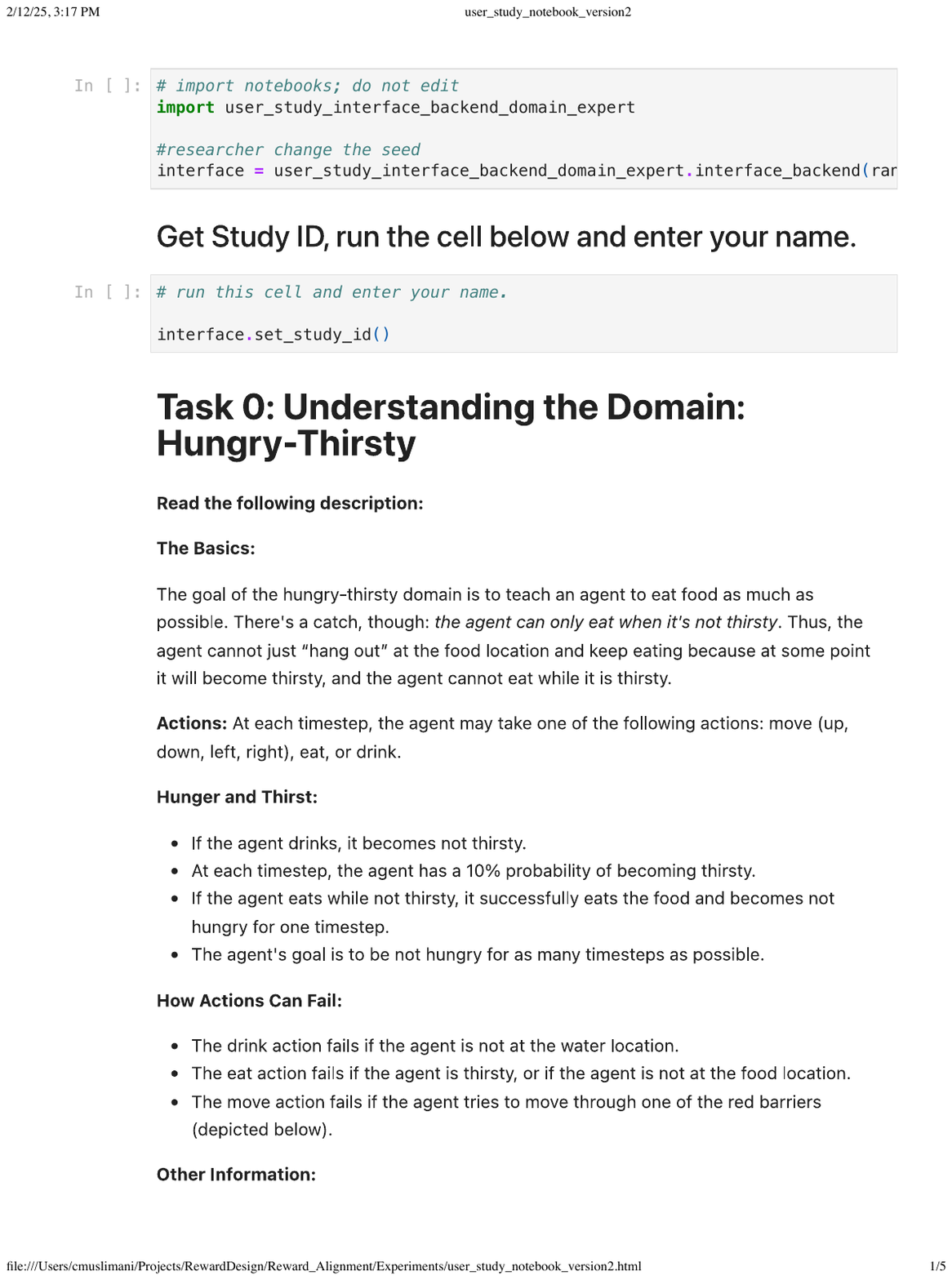}
    \caption{First page.}
    \label{fig:page1}
  \end{subfigure}
  \hfill
  \begin{subfigure}{0.48\textwidth}
    \centering
    \includegraphics[scale=0.33, trim=2cm 1cm 2cm 1cm, clip]{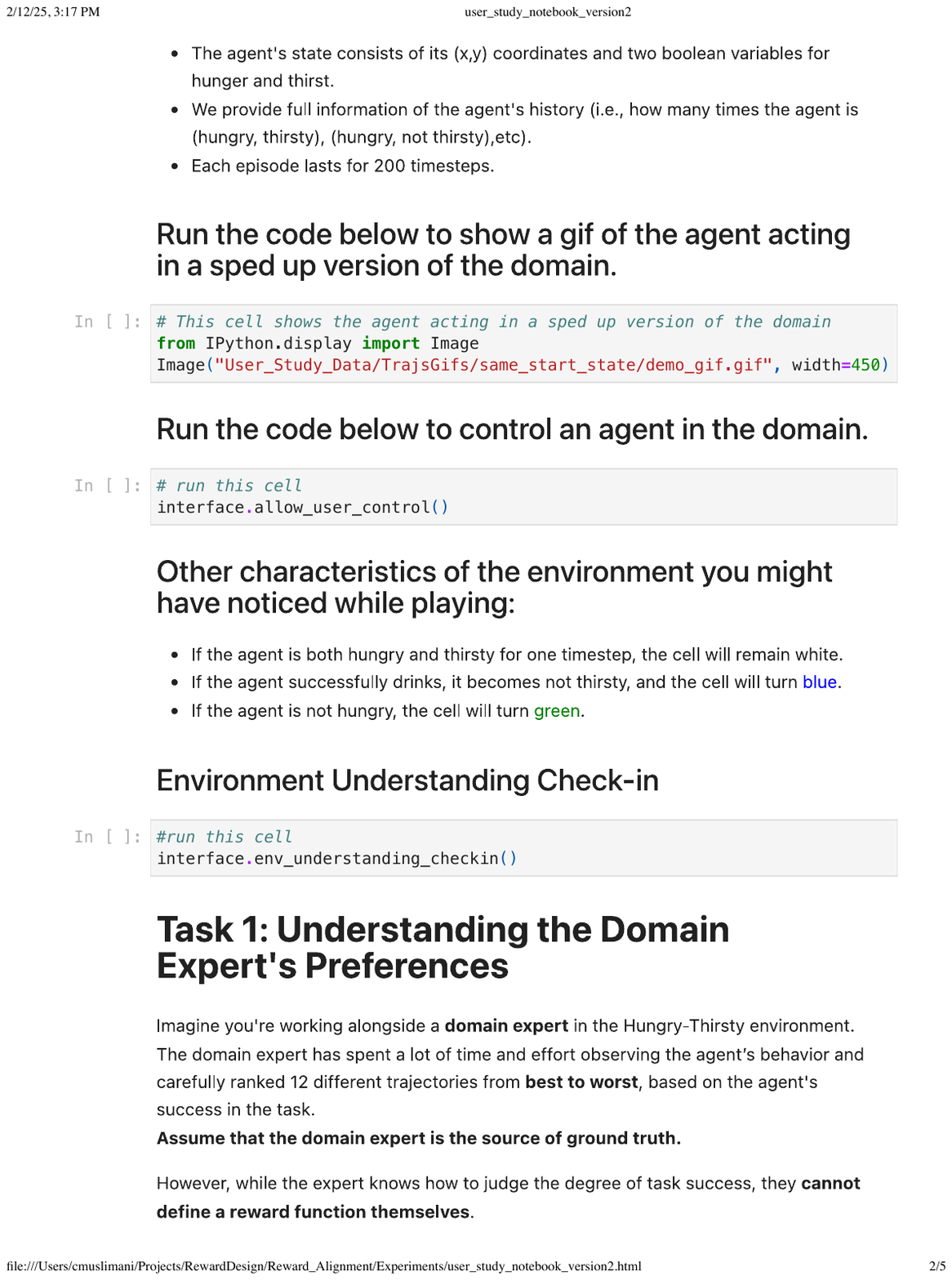}
    \caption{Second page.}
    \label{fig:page2}
  \end{subfigure}
  \caption{First and second pages of the UI in the human-subject study. }
  \label{fig:user_study_interface1}
\end{figure}

\begin{figure}[H]
  \centering
  \begin{subfigure}{0.48\textwidth}
    \centering
    \includegraphics[scale=0.4, trim=2cm 6cm 2cm 1cm, clip]{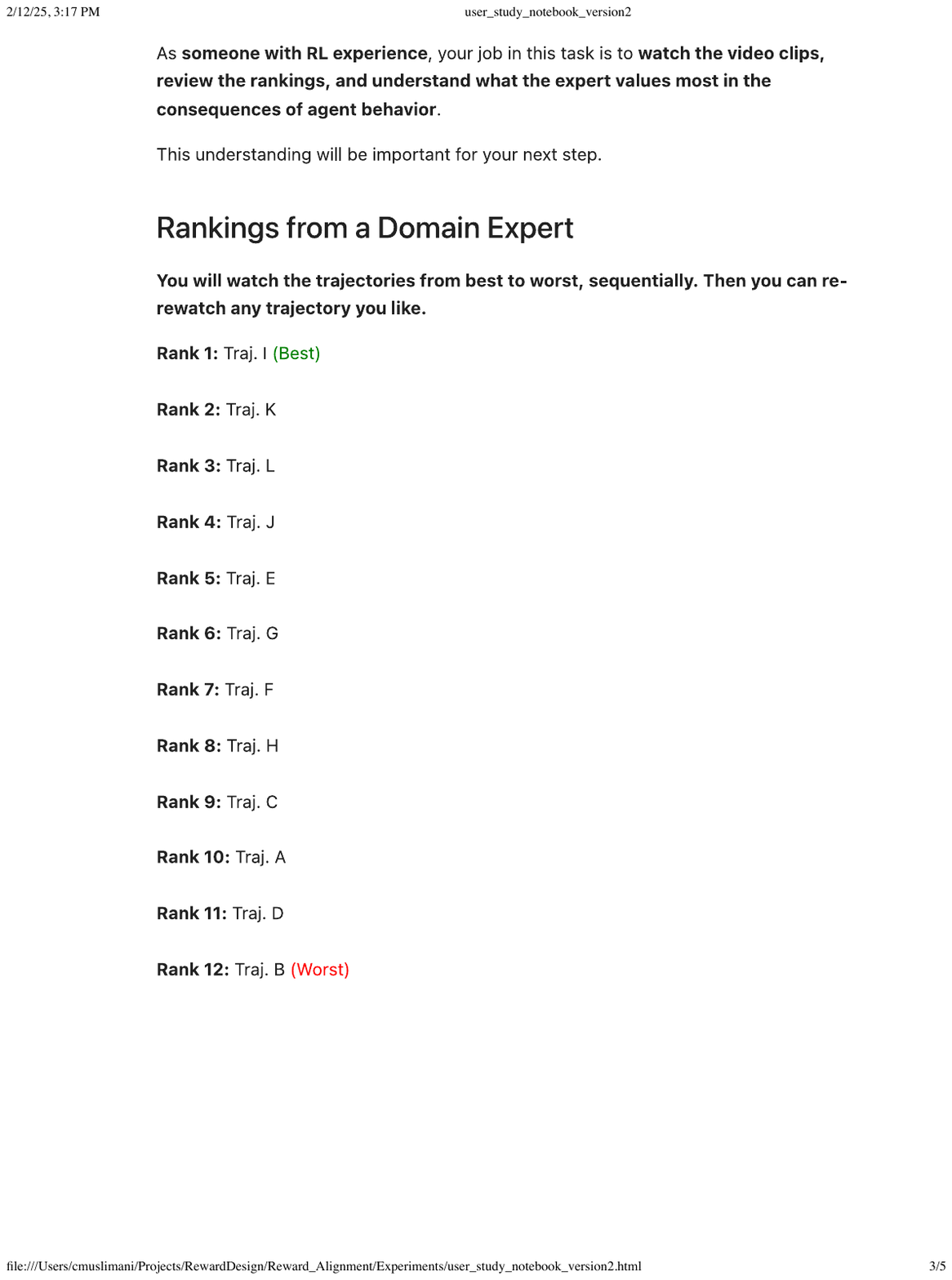}
    \caption{Third page.}
    \label{fig:page3}
  \end{subfigure}
  \hfill
  \begin{subfigure}{0.48\textwidth}
    \centering
    \includegraphics[scale=0.35,  trim=2cm 2cm 2cm 1cm, clip]{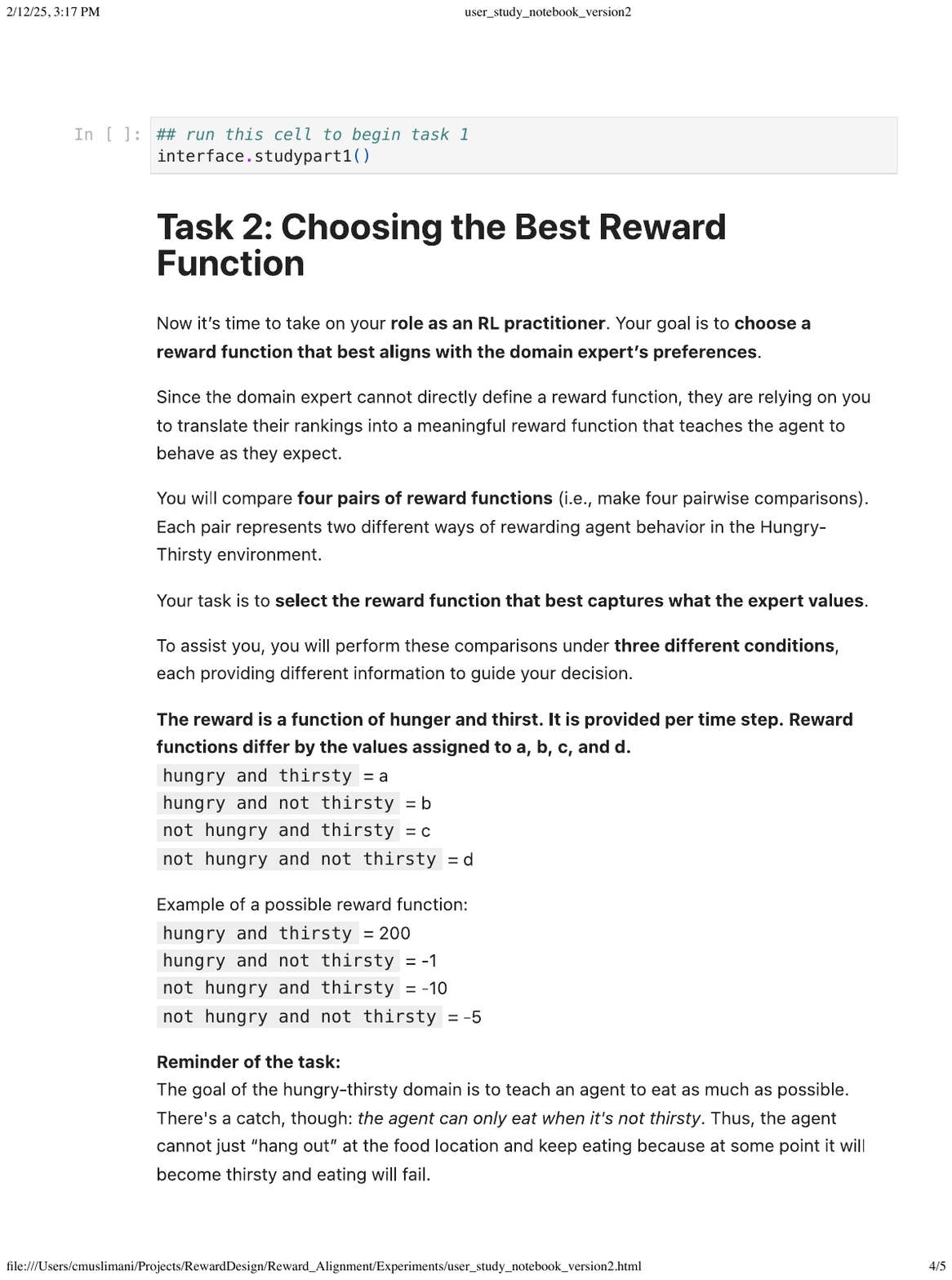}
    \caption{Fourth page.}
    \label{fig:page4}
  \end{subfigure}
  \caption{Third and fourth pages of the UI in the human-subject study. }
  \label{fig:user_study_interface2}
\end{figure}
\begin{figure}[H]
  \centering
\includegraphics[scale=0.35, trim = 0cm 5.5cm 0cm 1cm, clip]{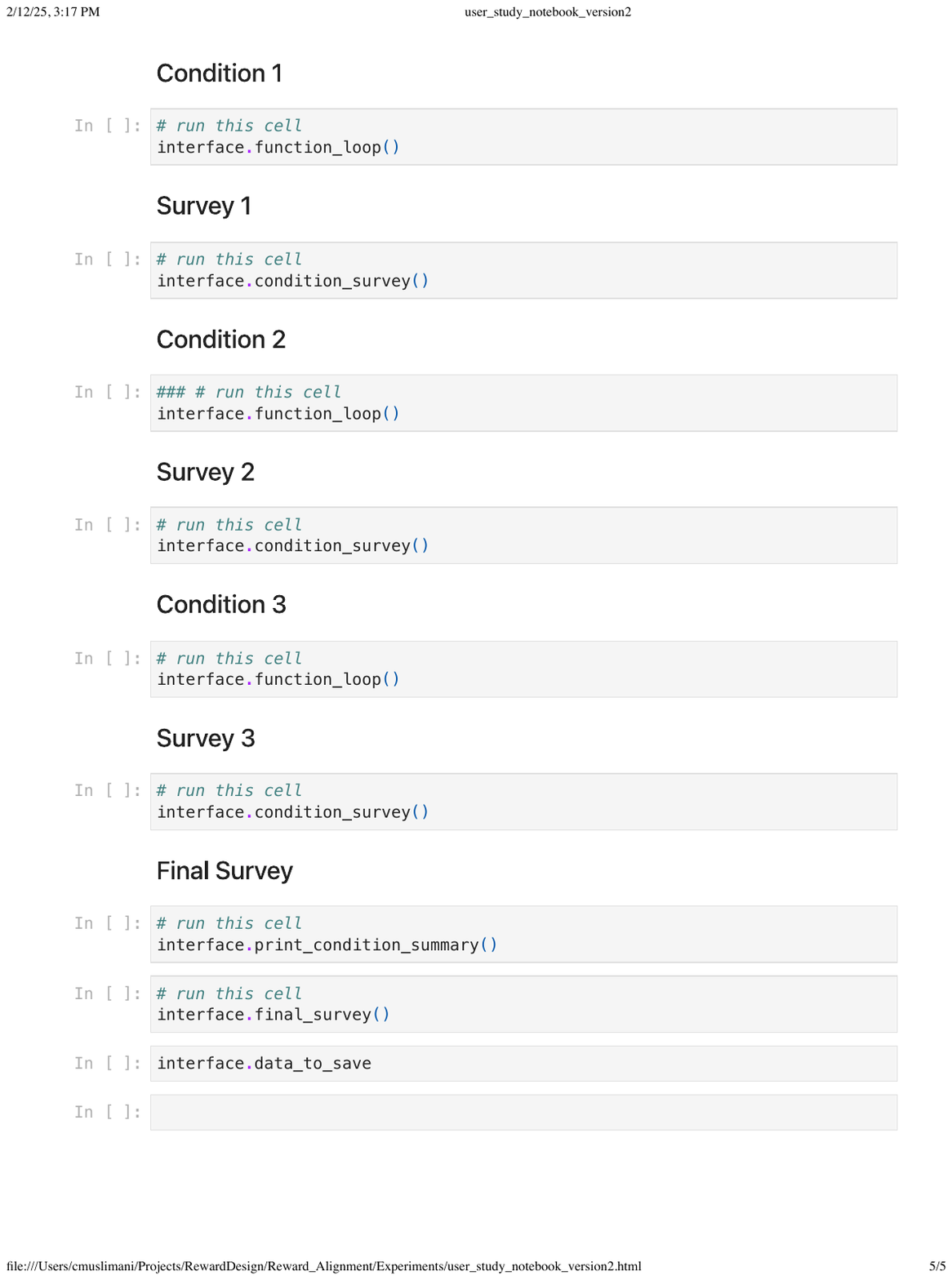}
  \caption{Fifth page of the UI in the human-subject study.}
  \label{fig:user_study_interface3}
\end{figure}

\begin{figure}[H]
  \centering
\includegraphics[scale=0.5, trim = 0cm 0cm 0cm 0cm, clip]{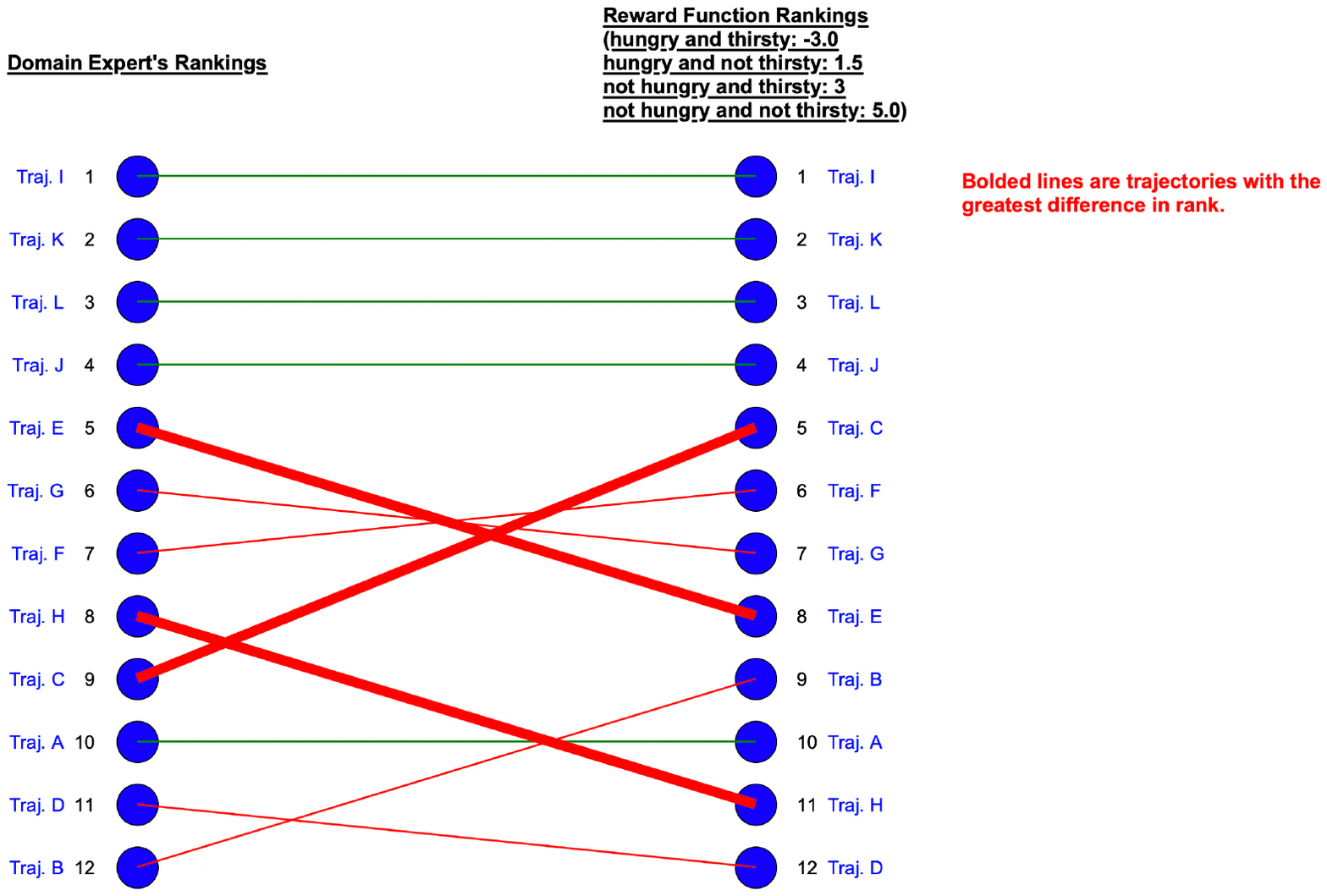}
  \caption{The visualization used in the Reward $+$ Alignment $+$ Visualization feedback condition. This shows how the preferences over trajectories differ between the domain expert and the reward function.}
  \label{fig:visualization}
\end{figure}

\begin{figure}[ht]
  \centering
  \begin{subfigure}{0.48\textwidth}
    \centering
    \includegraphics[scale=0.32, trim=0cm 0cm 2cm 0cm, clip]{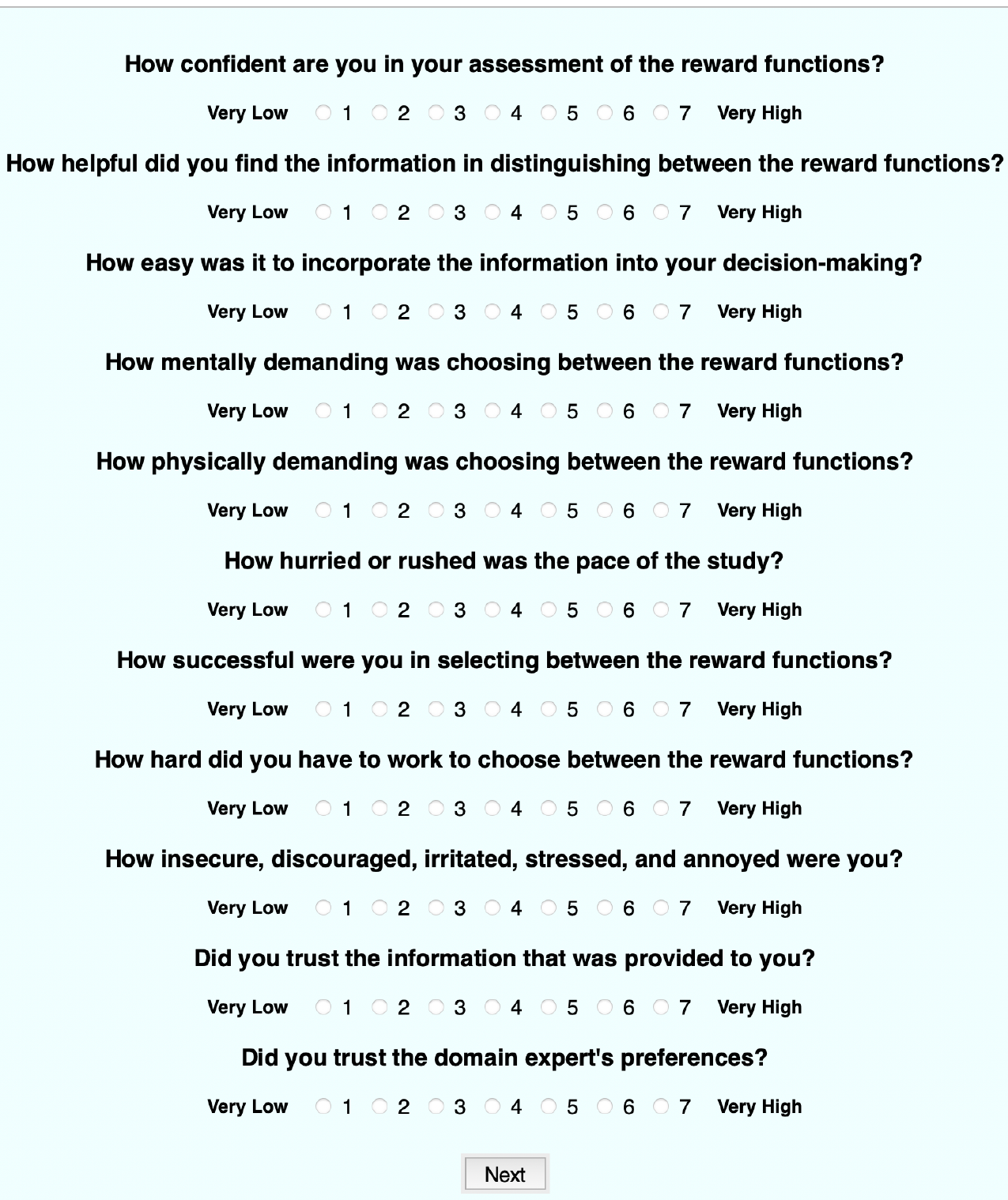}
    \caption{Modified NASA TLX survey given to participants after each condition.}
    \label{fig:nasa_survey}
  \end{subfigure}
  \hfill
  \begin{subfigure}{0.48\textwidth}
    \centering
    \includegraphics[scale=0.32,  trim=0cm 0cm 0cm 0cm, clip]{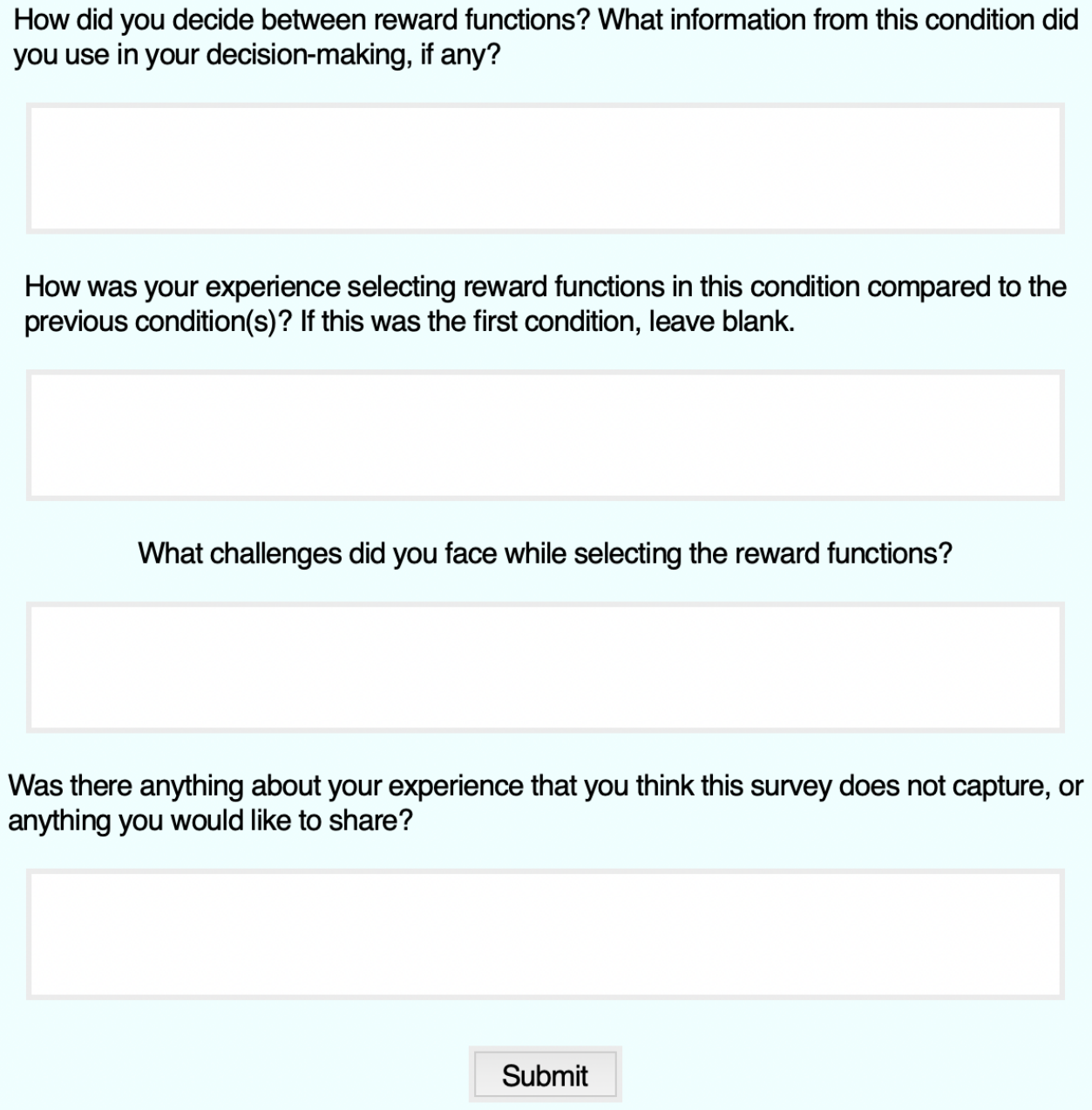}
    \caption{Short answer survey given to participants after each condition.}
    \label{fig:survey_shortanswer}
  \end{subfigure}
  \caption{Condition Experience Surveys.}
  \label{fig:condition_experiences_survey}
\end{figure}

\begin{figure}[H]
  \centering
  \begin{subfigure}{0.48\textwidth}
    \centering
    \includegraphics[scale=0.32, trim=0cm 0cm 2cm 0cm, clip]{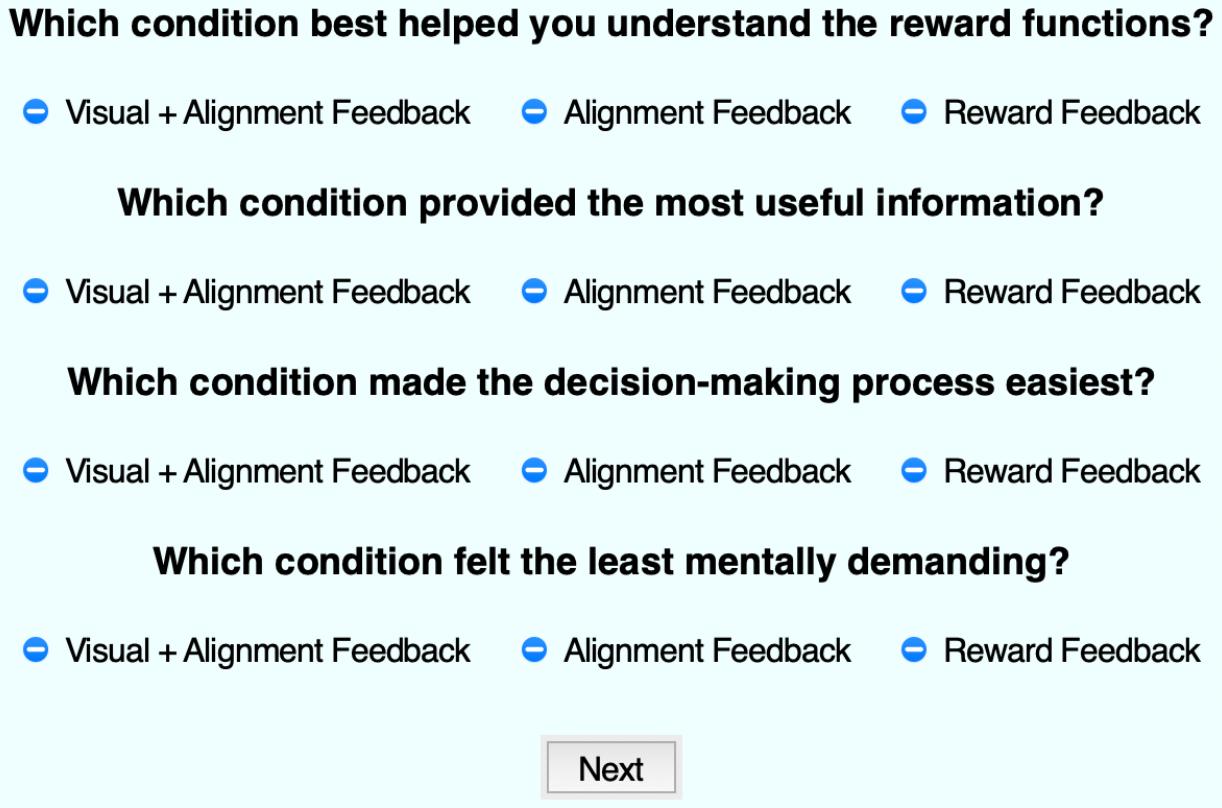}
    \caption{Final multiple choice survey given to participants to compare their experiences across the conditions. }
    \label{fig:comparison_survey_mc}
  \end{subfigure}
  \hfill
  \begin{subfigure}{0.48\textwidth}
    \centering
    \includegraphics[scale=0.32, trim=0cm 0cm 2cm 0cm, clip]{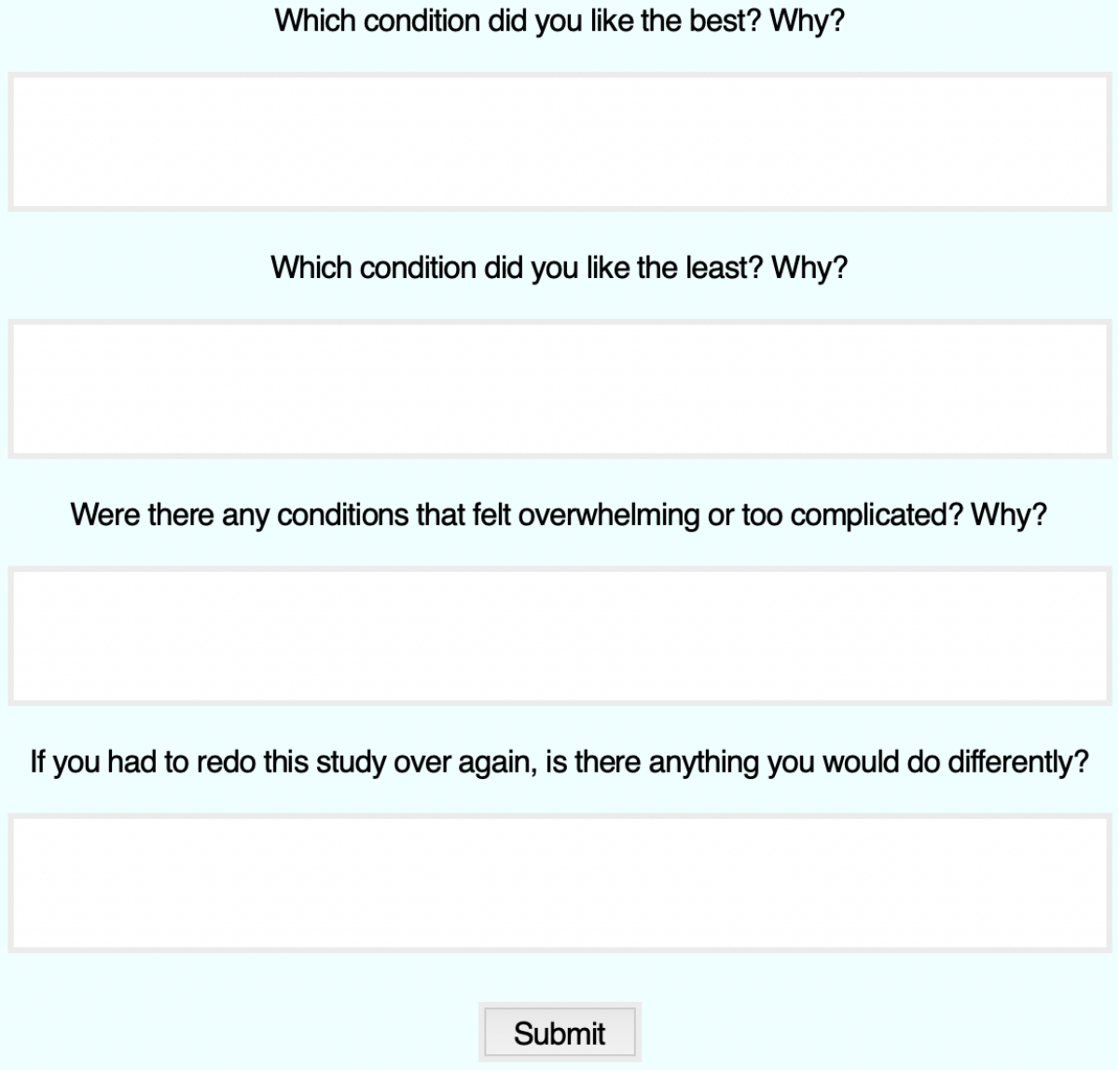}
    \caption{Final short answer survey given to participants to compare their experiences across the conditions.}
    \label{fig:comparison_shortanswer}
  \end{subfigure}
  \caption{Condition Comparison Surveys.}
  \label{fig:comparison_survey}
\end{figure}

Next, in Tables \ref{tab:final_return_comparison} and \ref{tab:auc_comparison}, we present the reward functions used in the human-subject study. In the study, the reward functions in the \textsc{Reward Function 1} column were compared with the corresponding reward functions in the \textsc{Reward Function 2} column. We report the mean final return and the area under the curve (AUC) (both calculated from the evaluation metric), along with the standard deviation (STD). Note that the reward functions in both tables are identical; the only distinction is the metric being displayed (e.g., return or AUC).
\begin{table}[H]
\centering
\resizebox{\textwidth}{!}{%
\begin{tabular}{|l|c|c|c|c|c|}
\hline
\textbf{\textsc{Reward Function 1}} & \textbf{\textsc{Final Return}} & \textbf{\textsc{Reward Function 2}} & \textbf{\textsc{Final Return}} \\ \hline
(-0.9, -0.7, -0.4, 1.1)         & 39.138 $\pm$ 38.619              & (-1, 0, 0.5, 1)                 & 30.832 $\pm$ 43.009              \\ \hline
(-3.7, 0.0, -3.1, 5.1)          & 33.495 $\pm$ 43.280              & (-3.0, 1.5, 3, 5.0)             & 1.079 $\pm$ 0.319                \\ \hline
(-0.9, -0.7, -0.4, 1.1)         & 39.138 $\pm$ 38.619              & (-0.05, 0.2, 1.0, 1.0)         & 1.608 $\pm$ 7.617                \\ \hline
(-3.6, 0.0, -3.1, 5.4)          & 35.485 $\pm$ 43.423              & (-5.8, 1.2, 3.6, 5.8)          & 1.555 $\pm$ 4.537                \\ \hline
(0, 0, 10, 10)                  & 49.709 $\pm$ 34.312              & (-0.05, 0.2, 1.0, 1.0)         & 1.608 $\pm$ 7.617                \\ \hline
(-5.0, 0, 3.25, 5.0)            & 31.582 $\pm$ 43.205              & (-5.0, 1.5, 3.25, 5.0)         & 1.225 $\pm$ 0.560                \\ \hline
(-0.5, -0.5, 10.0, 10.0)        & 51.755 $\pm$ 36.960              & (-0.05, 0.2, 1.0, 1.0)         & 1.608 $\pm$ 7.617                \\ \hline
(-0.4, -0.5, 0.0, 1.0)          & 44.338 $\pm$ 41.750              & (-0.2, 0.2, 0.5, 1.0)          & 1.217 $\pm$ 4.485                \\ \hline
(-5.0, 0.0, -2.5, 5.0)          & 29.738 $\pm$ 42.468              & (-5.0, 1.5, 3.25, 5.0)         & 1.225 $\pm$ 0.560                \\ \hline
(-1.0, -0.05, -0.25, 1.0)       & 37.488 $\pm$ 44.081              & (-5.0, 1.5, 3.25, 5.0)         & 1.225 $\pm$ 0.560                \\ \hline
(-3.75, 0.0, -3.0, 5.0)         & 33.127 $\pm$ 43.348              & (-5.0, 1.5, 3.25, 5.0)         & 1.225 $\pm$ 0.560                \\ \hline
(-1.0, -0.7, -0.5, 1.0)         & 35.772 $\pm$ 36.955              & (-0.05, 0.2, 1.0, 1.0)         & 1.608 $\pm$ 7.617                \\ \hline
\end{tabular}
}
\captionsetup{width=\textwidth}
\caption{This table shows the reward functions being compared in the reward selection aspect of the user study. We also report the mean final return $\pm$ STD.}
\label{tab:final_return_comparison}
\end{table}
\begin{table}[h!]
\centering
\resizebox{\textwidth}{!}{%
\begin{tabular}{|l|c|c|c|c|c|}
\hline
\textbf{\textsc{Reward Function 1}} & \textbf{\textsc{AUC}} & \textbf{\textsc{Reward Function 2}} & \textbf{\textsc{AUC}} \\ \hline
(-0.9, -0.7, -0.4, 1.1)         & 240041.894 $\pm$ 278901.770     & (-1, 0, 0.5, 1)                 & 203405.861 $\pm$ 293167.199 \\ \hline
(-3.7, 0.0, -3.1, 5.1)          & 219656.426 $\pm$ 296860.095     & (-3.0, 1.5, 3, 5.0)             & 10619.689 $\pm$ 1872.266   \\ \hline
(-0.9, -0.7, -0.4, 1.1)         & 240041.894 $\pm$ 278901.770     & (-0.05, 0.2, 1.0, 1.0)         & 11390.661 $\pm$ 17391.171   \\ \hline
(-3.6, 0.0, -3.1, 5.4)          & 227802.180 $\pm$ 296571.891     & (-5.8, 1.2, 3.6, 5.8)          & 13451.407 $\pm$ 15346.038   \\ \hline
(0, 0, 10, 10)                  & 320899.367 $\pm$ 251997.175     & (-0.05, 0.2, 1.0, 1.0)         & 11390.661 $\pm$ 17391.171   \\ \hline
(-5.0, 0, 3.25, 5.0)            & 206036.191 $\pm$ 295113.945     & (-5.0, 1.5, 3.25, 5.0)         & 11585.087 $\pm$ 3764.687    \\ \hline
(-0.5, -0.5, 10.0, 10.0)        & 349167.541 $\pm$ 281517.767     & (-0.05, 0.2, 1.0, 1.0)         & 11390.661 $\pm$ 17391.171   \\ \hline
(-0.4, -0.5, 0.0, 1.0)          & 284235.135 $\pm$ 304976.538     & (-0.2, 0.2, 0.5, 1.0)          & 10914.361 $\pm$ 15017.765   \\ \hline
(-5.0, 0.0, -2.5, 5.0)          & 192430.552 $\pm$ 284398.848     & (-5.0, 1.5, 3.25, 5.0)         & 11585.087 $\pm$ 3764.687    \\ \hline
(-1.0, -0.05, -0.25, 1.0)       & 254052.926 $\pm$ 304591.173     & (-5.0, 1.5, 3.25, 5.0)         & 11585.087 $\pm$ 3764.687    \\ \hline
(-3.75, 0.0, -3.0, 5.0)         & 216262.144 $\pm$ 295986.126     & (-5.0, 1.5, 3.25, 5.0)         & 11585.087 $\pm$ 3764.687    \\ \hline
(-1.0, -0.7, -0.5, 1.0)         & 217347.509 $\pm$ 256831.683     & (-0.05, 0.2, 1.0, 1.0)         & 11390.661 $\pm$ 17391.171   \\ \hline
\end{tabular}
}
\captionsetup{width=\textwidth}
\caption{This table shows the reward functions being compared in the reward selection aspect of the user study. We also report the mean AUC $\pm$ STD.}
\label{tab:auc_comparison}
\end{table}

In an earlier version of the user study, two participants had a slightly different UI design. The differences included some variations in the wording of the instructions. We also did not include a game-play session that allowed user control in the domain. Additionally, we did not ask participants whether they trusted the domain expert or the information being provided to them. There were also minor changes in the reward functions considered. Initially, we had a set of $13$ pairs of reward functions, from which we sampled $12$ per participant. However, for simplicity in data analysis, we decided to remove one pair. We also replaced one reward function pair with another, to make the reward functions being compared more distinct.

\section{Proofs}\label{sec:proofs_detailed}

\begin{lemma}\label{linear_transformations_expected_returns_preserves_preferences}
Given the infinite-horizon setting, if the expected returns under reward function \( r' \) are a positive linear transformation of the expected returns under reward function \( r \), with respect to all trajectory distributions, then the preference ordering over any two trajectory distributions \( \eta_i \) and \( \eta_j \) remains unchanged. Formally:
\[
\mathbb{E}_{\tau \sim \eta}[G_{r'}(\tau)] = \alpha \mathbb{E}_{\tau \sim \eta}[G_{r}(\tau)] + \beta \implies 
\Big(\eta_i \underset{(r, \gamma)}{\succsim} \eta_j \iff \eta_i \underset{(r', \gamma)}{\succsim} \eta_j\Big) \quad \forall \eta_i, \eta_j,
\]
where \( \alpha > 0 \) and \( \beta \) are constants and the expectations \( \mathbb{E}_{\tau \sim \eta}[G_{r}(\tau)] \) and \( \mathbb{E}_{\tau \sim \eta}[G_{r'}(\tau)] \) are taken over the same trajectory distributions.
\end{lemma}

\begin{proof}
    Let \( \eta_i, \eta_j \) be arbitrary trajectory distributions. Without loss of generality, assume that \mbox{\( \eta_i \underset{(r, \gamma)}{\succsim} \eta_j \)}  
    From Definition \ref{def:preferences_trajectory_distribution}, this implies:  
    \[
    \eta_i \underset{(r, \gamma)}{\succsim} \eta_j \iff \mathbb{E}_{\tau \sim \eta_i}[G_r(\tau)] \geq \mathbb{E}_{\tau \sim \eta_j}[G_r(\tau)]
    \]
    Define the difference in expected returns under \( r \) as:
    \begin{equation}\label{eq:delta_g}
        \Delta_{i,j} \overline{G}(r) \doteq \mathbb{E}_{\tau \sim \eta_i}[G_r(\tau)] - \mathbb{E}_{\tau \sim \eta_j}[G_r(\tau)]
    \end{equation}
    
    Now, consider the transformation of the expected return under \( r' \):
    \[
    \mathbb{E}_{\tau \sim \eta}[G_{r'}(\tau)] = \alpha \mathbb{E}_{\tau \sim \eta}[G_r(\tau)] + \beta,
    \]
    where \( \alpha > 0 \) and \( \beta \) are constants.  
    Define the corresponding difference under \( r' \):
    \begin{equation}\label{eq:delta_g_prime}
        \Delta_{i,j} \overline{G}(r') = \mathbb{E}_{\tau \sim \eta_i}[G_{r'}(\tau)] - \mathbb{E}_{\tau \sim \eta_j}[G_{r'}(\tau)]
    \end{equation}
    
Substitute the expressions for $\mathbb{E}_{\tau \sim \eta}[G_{r'}(\tau)]$, we get:
    \[
    \Delta_{i,j} \overline{G}(r') = \Big( \alpha \mathbb{E}_{\tau \sim \eta_i}[G_r(\tau)] + \beta \Big) - \Big( \alpha \mathbb{E}_{\tau \sim \eta_j}[G_r(\tau)] + \beta \Big)
    \]
    Simplify the terms and notice that \( \beta \) cancels out. We obtain:
    \[
    \Delta_{i,j} \overline{G}(r') = \alpha \Delta_{i,j} \overline{G}(r)
    \]

    Now, consider the two cases:

   \textbf{Case 1}: \( \eta_i \underset{(r, \gamma)}{\succ} \eta_j \)
   
     This means \( \Delta_{i,j} \overline{G}(r) > 0 \). Since \( \alpha > 0 \), we have:  
     \[
     \Delta_{i,j} \overline{G}(r') = \alpha \Delta_{i,j} \overline{G}(r) > 0
     \]
     As $\alpha$ is a positive constant, we conclude: \( \eta_i \underset{(r', \gamma)}{\succ} \eta_j \).  

   \textbf{Case 2}: \( \eta_i \underset{(r, \gamma)}{\sim} \eta_j \)
   
     This means \( \Delta_{i,j} \overline{G}(r) = 0 \). Apply the transformation to obtain: 
     \[
     \Delta_{i,j} \overline{G}(r') = \alpha \cdot 0 = 0
     \]
     Thus, \( \eta_i \underset{(r', \gamma)}{\sim} \eta_j \)

    Since both cases preserve the preference ordering, we conclude:
    \[
    \eta_i \underset{(r, \gamma)}{\succsim} \eta_j \iff \eta_i \underset{(r', \gamma)}{\succsim} \eta_j
    \]
\end{proof}

    \begin{lemma}\label{Necessity_lemma}[Necessity]
   In the infinite horizon setting, if two trajectory distributions \( \eta_i \in H(\mu_i) \) and \( \eta_j \in H(\mu_j) \) have different start-state distributions (\(\mu_i \neq \mu_j\)), then there exists a potential function \( \Phi \) such that:
    \[
    \eta_i \underset{(r, \gamma)}{\succsim}  \eta_j \textit{ and } \eta_i \underset{(r', \gamma)}{\prec} \eta_j.
    \]
    \end{lemma}

\begin{proof}\label{Necessity_lemma_proof}
Let \( \eta_i \in H(\mu_i) \) and \( \eta_j \in H(\mu_j) \)
be arbitrary trajectory distributions that have different start-state distributions (\(\mu_i \neq \mu_j\)). Without loss of generality assume that \( \eta_i \underset{(r, \gamma)}{\succsim} \eta_j \).  From Definition \eqref{def:preferences_trajectory_distribution}, this implies that: 
\[\mathbb{E}_{\tau\sim\eta_i} [G_r(\tau)] \ge \mathbb{E}_{\tau\sim\eta_j} [G_r(\tau)]\] We now analyze how the expected return changes under the potential-shaped reward function \( r' \). 

From Equation \eqref{eq:shaping_expected_return_with_phi} we have
\begin{align}
\mathbb{E}_{\tau\sim \eta_{i}(\tau)}[G_{r'}] &= 
\mathbb{E}_{\tau\sim \eta_{i}(\tau)}[G_{r}(\tau)] - 
\mathbb{E}_{s_{0}\sim \mu_{i}} \left[\Phi(s_0) \right], \label{eq:expected_return_phii} \\
\mathbb{E}_{\tau\sim \eta_{j}(\tau)}[G_{r'}] &= 
\mathbb{E}_{\tau\sim \eta_{j}(\tau)}[G_{r}(\tau)] - 
\mathbb{E}_{s_{0}\sim \mu_{j}} \left[\Phi(s_0) \right]. \label{eq:expected_return_phij}
\end{align}
Next, we define $\Delta_{i,j} \overline{G}(r), \Delta_{i,j} \overline{G}(r')$:
\begin{equation}\label{eq:delta_g2}
\Delta_{i,j} \overline{G}(r) \doteq \mathbb{E}_{\tau \sim \eta_i}[G_r(\tau)] - \mathbb{E}_{\tau \sim \eta_j}[G_r(\tau)]
\end{equation}
\begin{equation}\label{eq:delta_g_prime2}
\Delta_{i,j} \overline{G}(r) \doteq \mathbb{E}_{\tau \sim \eta_i}[G_{r'}(\tau)] - \mathbb{E}_{\tau \sim \eta_j}[G_{r'}(\tau)]
\end{equation}
Substitute Equations \eqref{eq:expected_return_phii} and \eqref{eq:expected_return_phij} into Equation \eqref{eq:delta_g_prime2} to obtain $\Delta_{i,j} \overline{G}(r')$:
\begin{align*}
\Delta_{i,j} \overline{G}(r') 
&= \Bigl(
\mathbb{E}_{\tau\sim \eta_{i}(\tau)}[G_{r}(\tau)] -  \mathbb{E}_{s_{0}\sim \mu_{i}}\left[\Phi(s_0)\right]\Bigr) - \Bigl(
\mathbb{E}_{\tau\sim \eta_{j}(\tau)}[G_{r}(\tau)] - 
\mathbb{E}_{s_{0}\sim \mu_{j}}\left[\Phi(s_0)\right]
\Bigr)
\end{align*}

Rearrange terms and substitute in the equation for $\Delta_{i,j} \overline{G}(r)$, Equation \eqref{eq:delta_g2}:
\begin{align}
\Delta_{i,j} \overline{G}(r') &= \Bigl(\mathbb{E}_{\tau\sim \eta_{i}(\tau)}[G_{r}(\tau)] - \mathbb{E}_{\tau\sim \eta_{j}(\tau)}[G_{r}(\tau)]\Bigr) - \Bigl(\mathbb{E}_{s_{0}\sim \mu_i}\left[\Phi(s_0)\right]  - \mathbb{E}_{s_{0}\sim \mu_j}\left[\Phi(s_0)\right]\Bigr) \nonumber \\
&= \Delta_{i,j} \overline{G}(r) - \Bigl(\mathbb{E}_{s_{0}\sim \mu_i}\left[\Phi(s_0)\right]  - \mathbb{E}_{s_{0}\sim \mu_j}\left[\Phi(s_0)\right]\Bigr)\label{eq:delta_g_with_phi}
\end{align}
Now we show that there exists a potential-based shaping function that will invert the preference ordering over trajectory distributions $\eta_i, \eta_j$, that is $ \exists \Phi: \mathcal{S} \to \mathbb{R}$ such that \( \eta_i \underset{(r', \gamma)}{\prec} \eta_j\).

From Definition \ref{def:preferences_trajectory_distribution}, it follows that:
\begin{align}
 \eta_i \underset{(r', \gamma)}{\prec} \eta_j \iff \Delta_{i,j} \overline{G}(r') < 0 \nonumber
\end{align}
This provides the necessary condition for the existence of such a shaping function. Now let $\Delta_{i,j}\overline{\Phi}$ be defined as:
\begin{equation}\label{delta_phi}
\Delta_{i,j}\overline{\Phi} \doteq \mathbb{E}_{s_{0}\sim \mu_i}\left[\Phi(s_0)\right] - \mathbb{E}_{s_{0}\sim \mu_j}\left[\Phi(s_0)\right]
\end{equation}

Combine Equations \eqref{eq:delta_g_with_phi} and \eqref{delta_phi} to get:
\begin{equation}
    \Delta_{i,j} \overline{G}(r') = \Delta_{i,j} \overline{G}(r) - \Delta_{i,j}\overline{\Phi}  \nonumber
\end{equation}
where $\Delta_{i,j} \overline{G}(r) > 0$, since \( \eta_i \underset{(r, \gamma)}{\succsim} \eta_j \). Now it is clear that: \begin{equation}
    \Delta_{i,j} \overline{G}(r') < 0 \iff \Delta_{i,j}\overline{\Phi} > \Delta_{i,j} \overline{G}(r)  \nonumber
\end{equation}

We now provide an example of a potential-based shaping function for which $\Delta_{i,j}\overline{\Phi} > \Delta_{i,j} \overline{G}(r)$.
To begin, let us partition the state space $\mathcal{S}$ into two subsets: (1) the set of states that are more probable under $\mu_i$ and (2) those that are more probable under $\mu_j$, denoted as $\mathcal{S}_{\mu_i > \mu_j}$ and $\mathcal{S}_{\mu_i \leq \mu_j}$, respectively:
\begin{align}
    \mathcal{S}_{\mu_i > \mu_j} &= \{s|s\in\mathcal{S}, \mu_i(s)>\mu_j(s)\}\label{eq:partition_ui_greater_uj} \\
    \mathcal{S}_{\mu_i \leq \mu_j} &= \{s|s\in\mathcal{S}, \mu_i(s)\leq\mu_j(s)\}\label{eq:partition_ui_lesser_uj}
\end{align}
Next, we define the potential-based shaping function $\Phi: \mathcal{S} \to \mathbb{R}$ as a piecewise function, which takes the following form:
\begin{equation}\label{eq:peicewise_function}
    \Phi(s) =
\begin{cases} 
\frac{\Delta_{i,j} \overline{G}(r) + \epsilon}{\int_{s\in\mathcal{S}_{\mu_i>\mu_j}}\mu_i(s)-\mu_j(s)}, & \text{if } s\in\mathcal{S}_{\mu_i>\mu_j} \\
0 & \text{if } s\in\mathcal{S}_{\mu_i\leq\mu_j} 
\end{cases}
\end{equation}
where $\epsilon\in\mathbb{R}, \epsilon > 0$. For this shaping function, we define the difference in expected values as:
\begin{align*}
    \Delta_{i,j}\overline{\Phi} &\doteq \mathbb{E}_{s_{0}\sim \mu_i}\left[\Phi(s_0)\right] - \mathbb{E}_{s_{0}\sim \mu_j}\left[\Phi(s_0)\right] 
\end{align*}
We can express the expectation in integral form and rearrange the terms:
\begin{align*}
    &= \int_{s\in\mathcal{S}}\mu_i(s)\Phi(s) - \int_{s\in\mathcal{S}}\mu_j(s)\Phi(s) \\
    &= \int_{s\in\mathcal{S}}\bigl(\mu_i(s) - \mu_j(s)\bigr)\Phi(s) 
\end{align*}
Decompose the integral over the two partitions defined in Equations \eqref{eq:partition_ui_greater_uj} and \eqref{eq:partition_ui_lesser_uj}. Notice that $ \int_{s\in\mathcal{S}_{\mu_i \leq \mu_j}}\bigl(\mu_i(s) - \mu_j(s)\bigr)\Phi(s)$ goes to $0$ by Equation \eqref{eq:peicewise_function}:
\begin{align*}
    &= \int_{s\in\mathcal{S}_{\mu_i > \mu_j}}\bigl(\mu_i(s) - \mu_j(s)\bigr)\Phi(s) + \int_{s\in\mathcal{S}_{\mu_i \leq \mu_j}}\bigl(\mu_i(s) - \mu_j(s)\bigr)\Phi(s) \\
    &= \int_{s\in\mathcal{S}_{\mu_i > \mu_j}}\bigl(\mu_i(s) - \mu_j(s)\bigr)\Phi(s) 
\end{align*}
Move $\Phi(s)$ outside of the integral as it is a constant by Equation \eqref{eq:peicewise_function} and simplify:
\begin{align*}
    &= \frac{\Delta_{i,j} \overline{G}(r) + \epsilon}{\int_{s\in\mathcal{S}_{\mu_i>\mu_j}}\mu_i(s)-\mu_j(s)} \int_{s\in\mathcal{S}_{\mu_i > \mu_j}}\mu_i(s) - \mu_j(s) \\
    &= \Delta_{i,j} \overline{G}(r) + \epsilon
\end{align*}
Hence $\Delta_{i,j}\overline{\Phi} = \Delta_{i,j} \overline{G}(r) + \epsilon > \Delta_{i,j} \overline{G}(r)$.
\end{proof}

\begin{theorem}\label{invariant_pos_linear_transformation}
 Given the infinite-horizon setting, the Trajectory Alignment Coefficient is invariant to positive linear transformations. 
\end{theorem}
\begin{proof}\label{invariant_pos_linear_transformation_proof}
    Let $\{\eta_i, \eta_j\} \in \upsilon_h(D_h)$ be an arbitrary pair of trajectory distributions compared in the human preference dataset.  Without loss of generality assume that \( \eta_i^r \succsim \eta_j^r \). From Defintion \eqref{def:preferences_trajectory_distribution}, this implies that
\[
\mathbb{E}_{\tau \sim \eta_i}[G_r(\tau)] \ge \mathbb{E}_{\tau \sim \eta_j}[G_r(\tau)].
\]

We now analyze how the expected return changes under the reward function \( r' \). 
The expected return under the reward function \( r \) is:
\begin{align}
\mathbb{E}_{\tau \sim \eta}[G_r(\tau)] &= \mathbb{E}_{\tau \sim \eta} \left[ \sum_{t=0}^\infty \gamma^t r(s_t, a_t, s_{t+1}) \right]
\end{align}
and the expected return under the shaped reward function \( r' \) is:
\begin{align}
\mathbb{E}_{\tau \sim \eta}[G_{r'}(\tau)] &= \mathbb{E}_{\tau \sim \eta} \left[ \sum_{t=0}^\infty \gamma^t r'(s_t, a_t, s_{t+1}) \right]\label{eq:expected_return_r_prime1}
\end{align}

Substitute \( r'(s, a, s') = \alpha \cdot r(s, a, s') + \beta \) into Equation \eqref{eq:expected_return_r_prime1}, we obtain:
\begin{align*}
\mathbb{E}_{\tau \sim \eta}[G_{r'}(\tau)] &= \mathbb{E}_{\tau \sim \eta} \left[ \sum_{t=0}^{\infty} \gamma^t (\alpha \cdot r(s_t, a_t, s_{t+1}) + \beta) \right] \\
&= \alpha \cdot \mathbb{E}_{\tau \sim \eta} \left[ \sum_{t=0}^{\infty} \gamma^t r(s_t, a_t, s_{t+1}) \right] + \mathbb{E}_{\tau \sim \eta} \left[ \sum_{t=0}^{\infty} \gamma^t \beta \right]
\end{align*}

Since \( \beta \) is a constant, the expectation simplifies as follows:
\begin{align*}
\mathbb{E}_{\tau \sim \eta}[G_{r'}(\tau)] &= \alpha \cdot \mathbb{E}_{\tau \sim \eta} [G_r(\tau)] + \beta \sum_{t=0}^{\infty} \gamma^t \\
&= \alpha \cdot \mathbb{E}_{\tau \sim \eta} [G_r(\tau)] + \frac{\beta}{1 - \gamma}
\end{align*}

As $\mathbb{E}_{\tau \sim \eta}[G_{r'}(\tau)]$ is positive linear transformation of $\mathbb{E}_{\tau \sim \eta}[G_{r}(\tau)]$, we apply Lemma \eqref{linear_transformations_expected_returns_preserves_preferences} to get:
\[\big(\eta_i \underset{(r, \gamma)}\succsim \eta_j \iff \eta_i \underset{(r', \gamma)}\succsim \eta_j\big) \]
Thus, from Equation \eqref{eq:tac} and Definition \eqref{tac_invariance}, we conclude that: \[\sigma_{\text{TAC}}(D_h, D_{r, \gamma}) = \sigma_{\text{TAC}}(D_h, D_{r', \gamma})\] 

\end{proof}

\section{Environment Details}\label{sec:env_details}
We use a modified Hungry-Thirsty domain \citep{singh2009rewards}, see Figure \ref{fig:hungry_thirsty_pic}. The agent's start state is randomly placed at the beginning of each episode, while the food and water locations are randomly assigned per environment configuration (e.g., per run/seed). 
Lastly, reward functions take the form:
\begin{equation}
\begin{split}
    r(\text{hungry, thirsty}) = a \\
    r(\text{hungry, not thirsty}) = b \\
    r(\text{not hungry, thirsty}) = c \\
    r(\text{not hungry, not thirsty}) = d\\ \nonumber
    \end{split}
\end{equation}

\begin{figure}[H]
  \centering
 \includegraphics[width=0.4\textwidth]{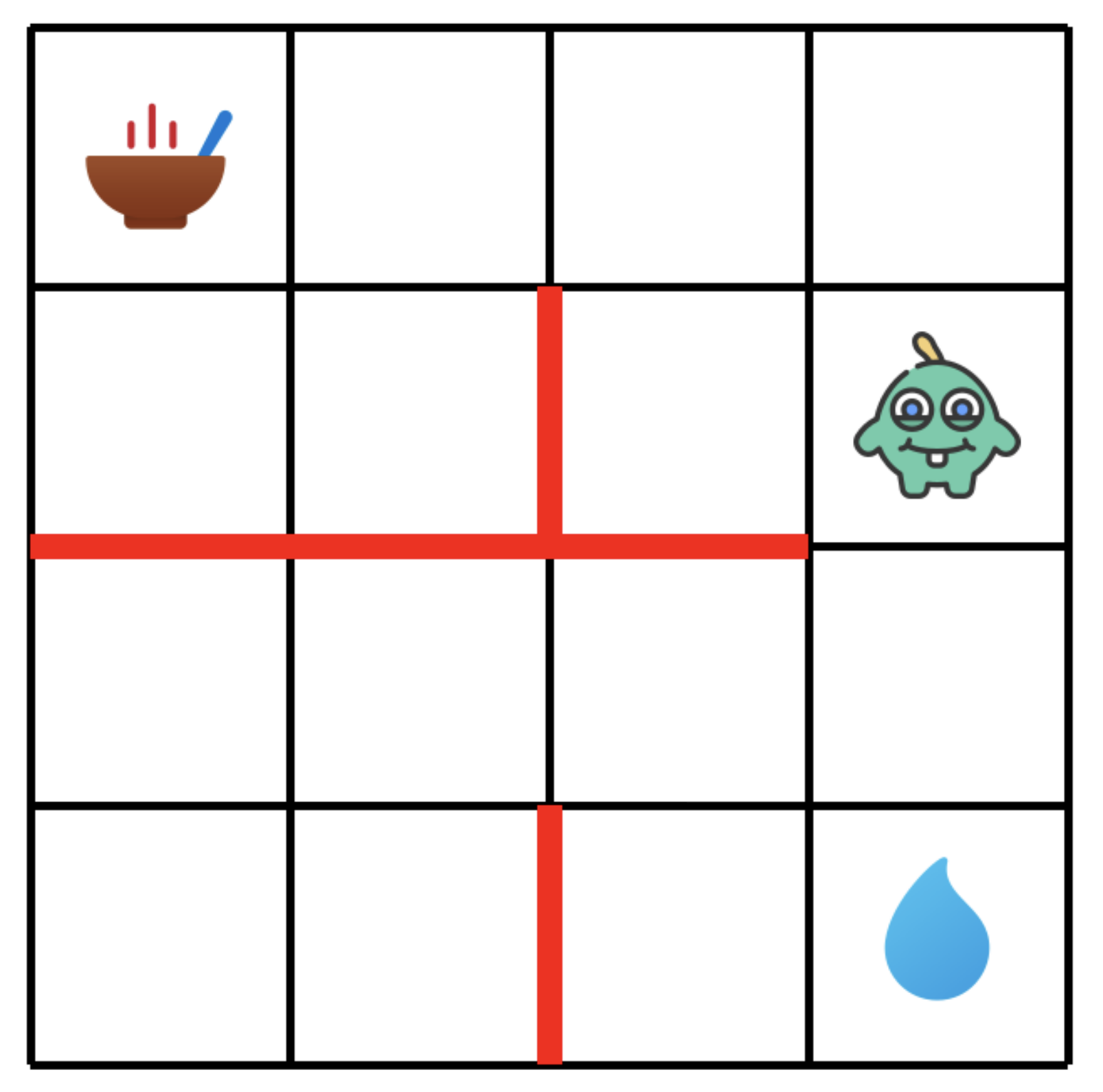}
  \caption{Hungry-Thirsty Environment}\label{fig:hungry_thirsty_pic}
\end{figure}

\end{document}